\documentclass[twoside]{article}

\usepackage[accepted]{aistats2020}
\usepackage[round]{natbib}

\usepackage[utf8]{inputenc} 
\usepackage[T1]{fontenc}    
\usepackage{lmodern}
\usepackage{url}            
\usepackage{booktabs, makecell}       
\usepackage{amsfonts}       
\usepackage{nicefrac}       
\usepackage{microtype}      

\usepackage[
            CJKbookmarks=true,
            bookmarksnumbered=true,
            bookmarksopen=true,
            colorlinks=true,
            citecolor=red,
            linkcolor=blue,
            anchorcolor=red,
            urlcolor=blue
            ]{hyperref}
  
  \renewcommand{\footnoterule}{%
  \kern -3pt
  \hrule width \textwidth height 1pt
  \kern 2pt
} 
  
\usepackage{amsmath,amssymb,amsthm}
\usepackage{prettyref,xspace,graphicx}
\usepackage{dsfont,tikz}
\usepackage{authblk}
\usepackage{bbding}

\usepackage{pifont}
\usepackage{algorithm}

\usepackage{subcaption}

\usepackage{epstopdf}

\usetikzlibrary{tikzmark,fit,shapes.geometric}

\newrefformat{cond}{Condition~\ref{#1}}
\newrefformat{eq}{Eq.~(\ref{#1})}
\newrefformat{thm}{Theorem~\ref{#1}}
\newrefformat{th}{Theorem~\ref{#1}}
\newrefformat{chap}{Chapter~\ref{#1}}
\newrefformat{sec}{Section~\ref{#1}}
\newrefformat{algo}{Algorithm~\ref{#1}}
\newrefformat{fig}{Figure~\ref{#1}}
\newrefformat{tab}{Table~\ref{#1}}
\newrefformat{rmk}{Remark~\ref{#1}}
\newrefformat{clm}{Claim~\ref{#1}}
\newrefformat{def}{Definition~\ref{#1}}
\newrefformat{cor}{Corollary~\ref{#1}}
\newrefformat{lmm}{Lemma~\ref{#1}}
\newrefformat{prop}{Proposition~\ref{#1}}
\newrefformat{pr}{Proposition~\ref{#1}}
\newrefformat{app}{Appendix~\ref{#1}}
\newrefformat{prob}{Problem~\ref{#1}}
\newrefformat{ques}{Question~\ref{#1}}
\newrefformat{note}{Note~\ref{#1}}
\newrefformat{assump}{Assumption~\ref{#1}}
\newrefformat{issue}{Issue ~\ref{#1}}
\newrefformat{fix}{Fix ~\ref{#1}}

\newcommand{\reals}{\mathbb{R}}
\newcommand{\naturals}{\mathbb{N}}

\newcommand{\Expect}{\mathbb{E}}

\newcommand{\prob}[1]{\mathbb{P}\left[#1\right]}

\newcommand{\inner}[2]{\langle {#1}, {#2}  \rangle}

\newcommand{\calE}{\mathcal{E}}

\newcommand{\calN}{\mathcal{N}}

\newcommand{\calP}{\mathcal{P}}
\newcommand{\calQ}{\mathcal{Q}}

\newcommand{\calS}{\mathcal{S}}
\newcommand{\calT}{\mathcal{T}}

\newcommand{\ba}{a}

\newcommand{\be}{e}

\newcommand{\bh}{h}

\newcommand{\bu}{u}
\newcommand{\bv}{v}
\newcommand{\bw}{w}
\newcommand{\bx}{x}
\newcommand{\by}{y}

\newcommand{\bA}{A}

\newcommand{\bD}{D}
\newcommand{\bE}{E}

\newcommand{\bI}{I}

\newcommand{\bM}{M}

\newcommand{\bP}{P}

\newcommand{\bW}{W}

\newcommand{\norm}[1]{\left \|#1\right \|}

\newcommand{\valid}{\mathrm{valid}}

\newcommand{\sym}{\mathrm{sym}}
\newcommand{\define}{\triangleq}

\newcommand{\pth}[1]{\left( #1 \right)}
\newcommand{\qth}[1]{\left[ #1 \right]}
\newcommand{\sth}[1]{\left\{ #1 \right\}}

\newcommand{\eg}{e.g.\xspace}
\newcommand{\ie}{i.e.\xspace}
\newcommand{\iid}{i.i.d.\xspace}

\newtheorem{theorem}{Theorem}
\newtheorem{lemma}{Lemma}

\newtheorem{remark}{Remark}

\newtheorem{condition}{Condition}

\theoremstyle{definition}
\newtheorem{definition}{Definition}
\newtheorem{example}{Example}

\newtheorem{question}{Question}

\newcommand{\Id}{\mathrm{Id}}

\begin{document}

%

%

\twocolumn[

\aistatstitle{Learning in Gated Neural Networks}
\aistatsauthor{Ashok Vardhan Makkuva$^{*}$ \And Sreeram Kannan$^{\dagger}$ \And  Sewoong Oh$^{\dagger}$ \And Pramod Viswanath$^{*}$ }

\aistatsaddress{$^{*}$University of Illinois at Urbana-Champaign \And  $^{\dagger}$University of Washington} ]

\begin{abstract}
Gating is a key feature in modern neural networks including LSTMs, GRUs and sparsely-gated deep neural networks. The backbone of such gated networks is a mixture-of-experts layer, where several experts make regression decisions and gating controls how to weigh the decisions in an input-dependent manner. Despite having such a prominent role in both modern and classical machine learning, very little is understood about parameter recovery of mixture-of-experts since gradient descent and EM algorithms are known to be stuck in local optima in such models.

In this paper, we perform a careful analysis of the optimization landscape and show that with appropriately designed loss functions, gradient descent can indeed learn the parameters of a MoE accurately. A key idea underpinning our results is the design of two {\em distinct} loss functions, one for recovering the expert parameters and another for recovering the gating parameters. We demonstrate the first sample complexity results for parameter recovery in this model for any algorithm and demonstrate significant performance gains over standard loss functions in numerical experiments. 
\end{abstract}

\begin{figure*}[t]
    \centering
    \begin{subfigure}[b]{0.32\textwidth}
        \includegraphics[width=\textwidth]{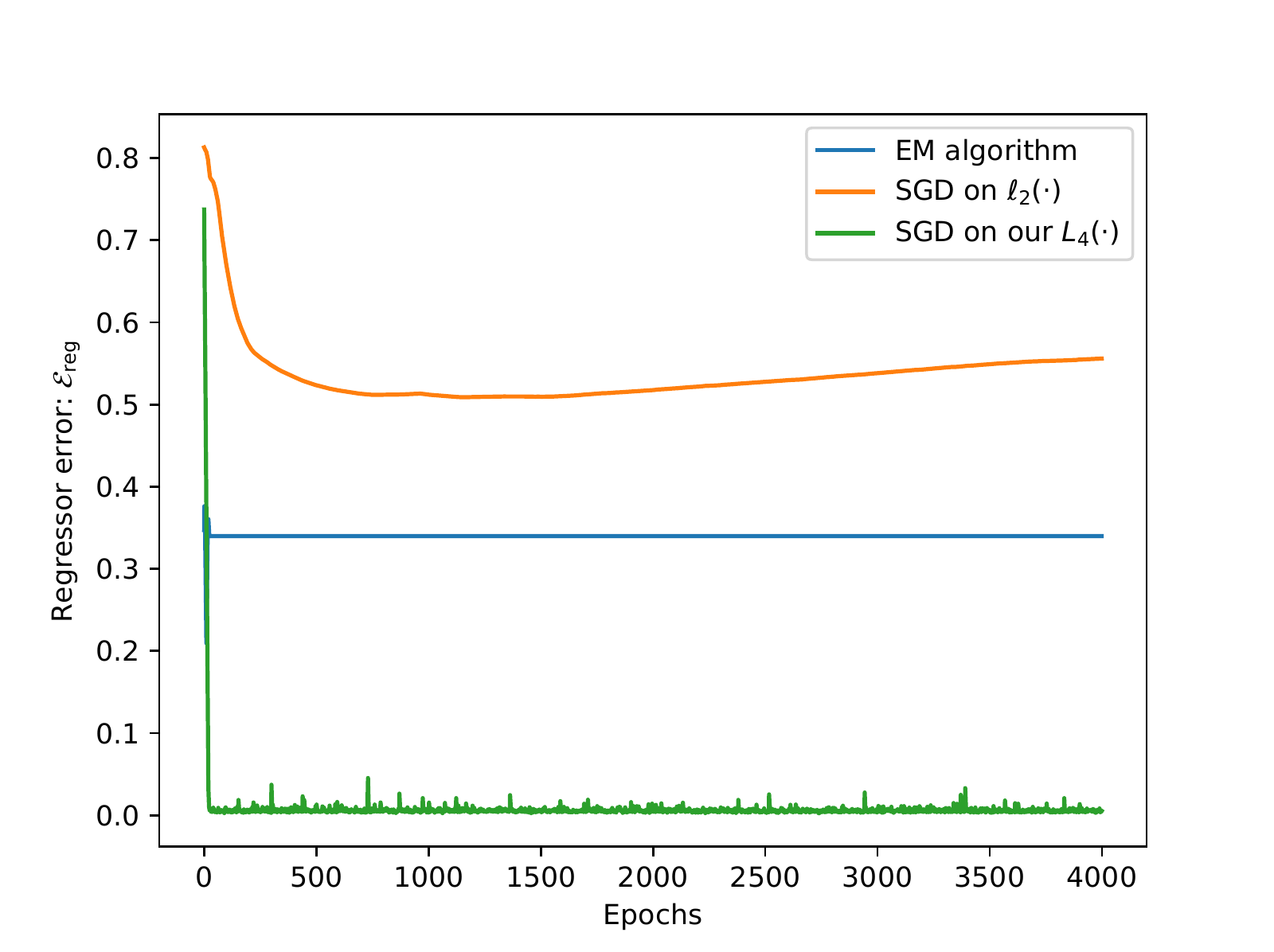}
        \caption{Regressor error}
        \label{fig:all_regressor_error}
    \end{subfigure}
    \begin{subfigure}[b]{0.32\textwidth}
        \includegraphics[width=\textwidth]{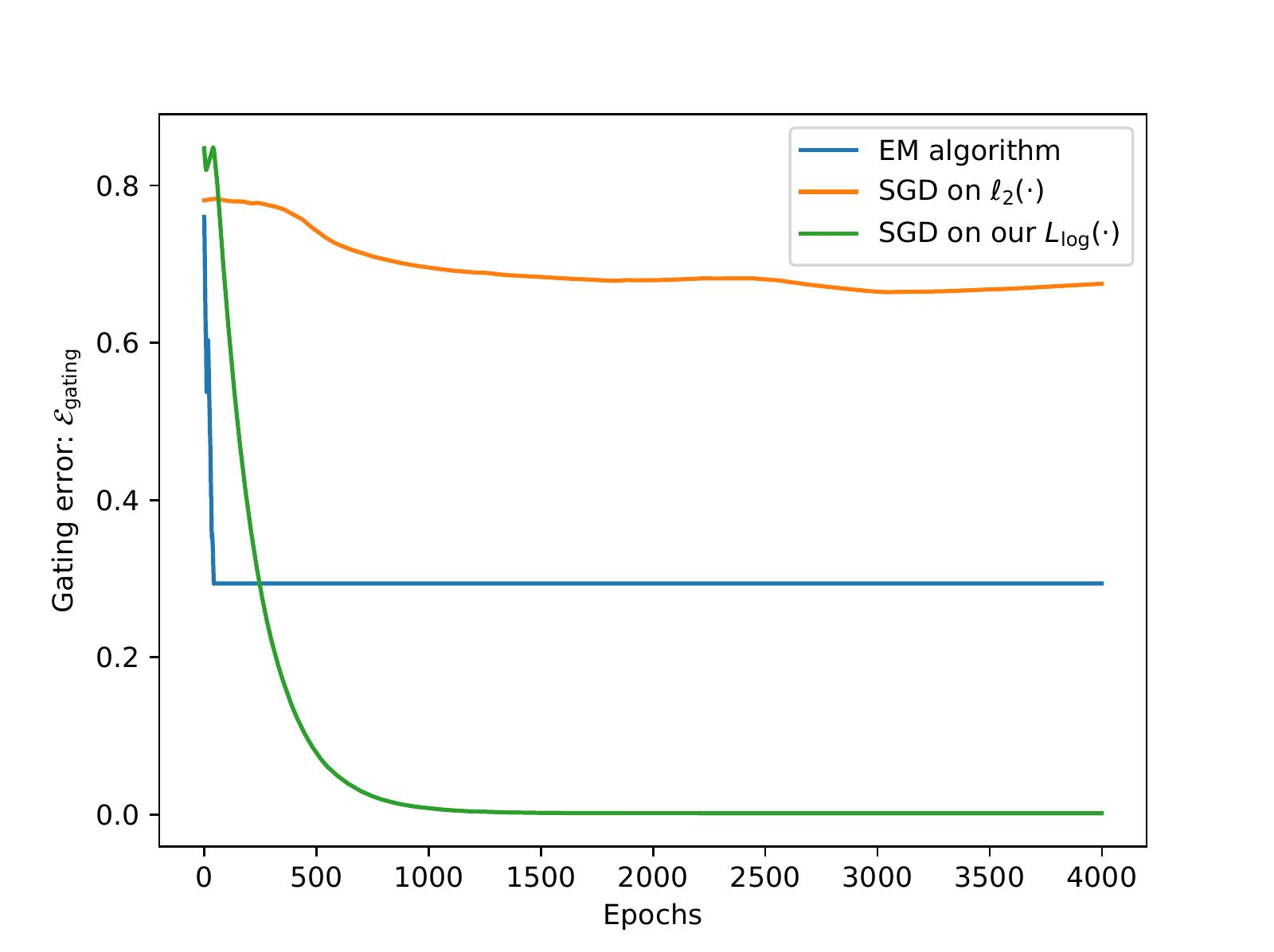}
        \caption{Gating error}
        \label{fig:all_classifier_error}
    \end{subfigure}
    ~ 
    \begin{subfigure}[b]{0.32\textwidth}
        \includegraphics[width=\textwidth]{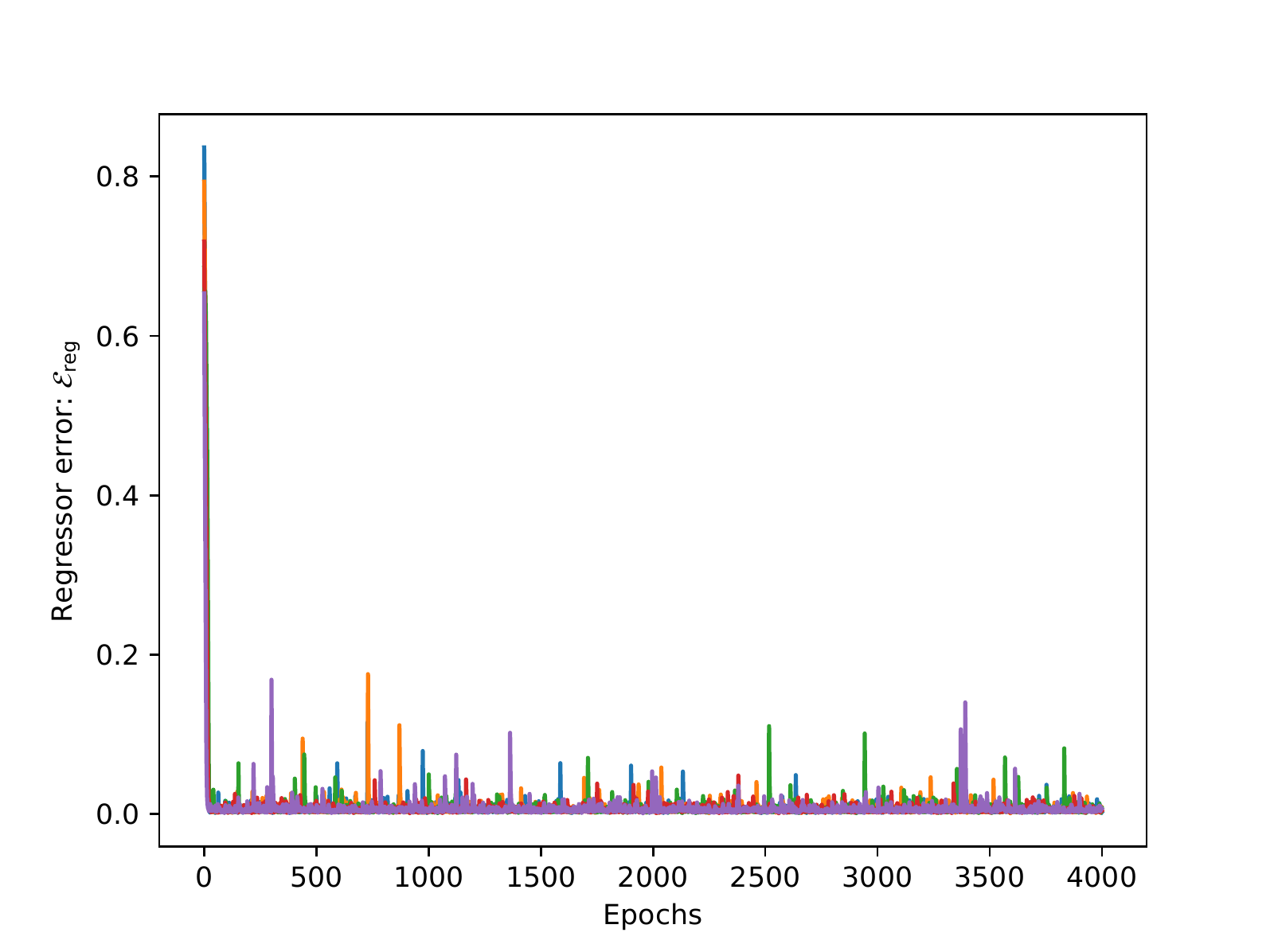}
        \caption{$L_4(\cdot)$ over different initializations}
        \label{fig:our_loss_multiple_runs}
    \end{subfigure}

\caption{ Our proposed losses $L_4$ (defined in \prettyref{eq:tensorloss}) and $L_{\mathrm{log}}$ (defined in \prettyref{eq:logloss}) to learn the respective regressor and gating parameters of a MoE model in \prettyref{eq:kmoe} achieve much better empirical results than the standard methods. }
\label{fig:regresslearn}
\end{figure*}

\section{Introduction}
In recent years, \emph{gated recurrent neural networks (RNNs)} such as LSTMs and GRUs have shown remarkable successes in a variety of challenging machine learning tasks such as machine translation, image captioning, image generation, hand writing generation, and speech recognition \citep{machinetranslation,imagecaption, speechrec,imagegen,handwritten}. A key interesting aspect and an important reason behind the success of these architectures is the presence of a \emph{gating mechanism} that dynamically controls the flow of the past information to the current state at each time instant. In addition, it is also well known that these gates prevent the vanishing (and exploding) gradient problem inherent to traditional RNNs \citep{hochreiter1997long}.

Surprisingly, despite their widespread popularity, there is very little theoretical understanding of these gated models. In fact, basic questions such as learnability of the parameters still remain open. Even for the simplest vanilla RNN architecture, this question was open until the very recent works of \cite{allen2018convergence} and \cite{allen2019can}, which provided the first theoretical guarantees of SGD for vanilla RNN models in the presence of non-linear activations. While this demonstrates that the theoretical analysis of these simpler models has itself been a challenging task, gated RNNs have an additional level of complexity in the form of gating mechanisms, which further enhances the difficulty of the problem. This motivates us to ask the following question:
\begin{question}
Given the complicated architectures of LSTMs/GRUs, can we find analytically tractable sub-structures of these models?
\end{question}

We believe that addressing the above question can provide new insights into a principled understanding of gated RNNs. In this paper, we make progress towards this and provide a positive answer to the question. In particular, we make a non-trivial connection that a GRU (gated recurrent unit) can be viewed as a time-series extension of a basic building block, known as \emph{Mixture-of-Experts (MoE)} \citep{JacJor,jor94}. In fact, much alike LSTMs/GRUs, MoE is itself a widely popular gated neural network architecture and has found success in a wide range of applications \citep{gpmoe,svmmoe,rasmussen2002infinite,YWG12,MaEb14,hmegp,deepmoe}. 
In recent years, there is also a growing interest in the fiels of natural language processing and computer vision to build complex neural networks incorporating MoE models to address challenging tasks such as machine translation \citep{gross2017hard, SMMD+17}. Hence the \emph{main goal} of this paper is to study MoE in close detail, especially with regards to learnability of its parameters.

The canonical MoE model is the following: Let $k \in \naturals$ denote the number of mixture components (or equivalently neurons). Let $\bx \in \reals^d$ be the input vector and $y \in \reals$ be the corresponding output. Then the relationship between $\bx$ and $y$ is given by: 
\begin{align}
y= \sum_{i=1}^k z_i \cdot g(\inner{\ba_i^\ast}{\bx}) + \xi, \quad \xi  \sim  \calN(0,\sigma^2),
\label{eq:kmoe}
\end{align}
where $g: \reals \to \reals$ is a non-linear activation function, $\xi$ is a Gaussian noise independent of $x$ and the latent Bernoulli random variable $z_i \in \{0,1\}$ indicates which expert has been chosen. In particular, only a single expert is active at any time, \ie $ \sum_{i=1}^k z_i =1$, and their probabilities are modeled by a soft-max function:
\begin{align*}
 \prob{z_i=1|\bx}=\frac{e^{\inner{\bw_i^\ast}{\bx}}}{\sum_{j=1}^k e^{\inner{\bw_j^\ast}{\bx}}}.
\end{align*}
Following the standard convention \citep{makkuva2018breaking,JacJor}, we refer to the vectors $\ba_i^\ast$ as \textit{regressors}, the vectors $\bw_i^\ast$ as either \textit{classifiers} or \textit{gating parameters}, and without loss of generality, we assume that $\bw_k^\ast=0$.

 Belying the canonical nature, and significant research effort, of the MoE model, the topic of learning MoE parameters is very poorly theoretically understood. In fact, the task of learning the parameters of a MoE, \ie $\ba_i^\ast$ and $\bw_i^\ast$, with provable guarantees is a long standing open problem for more than two decades \citep{SedghiA14a}. One of the key technical difficulties is that in a MoE, there is an inherent \emph{coupling} between the regressors $a_i^\ast$ and the gating parameters $w_i^\ast$, as can be seen from \prettyref{eq:kmoe}, which makes the problem challenging \citep{ho2019convergence}. In a recent work \citep{makkuva2018breaking}, the authors provided the first consistent algorithms for learning MoE parameters with theoretical guarantees. In order to tackle the aforementioned coupling issue, they proposed a clever scheme to first estimate the regressor parameters $\ba_i^\ast$ and then estimating the gating parameters $\bw_i^\ast$ using a combination of spectral methods and the EM algorithm. However, a major draw back is that this approach requires specially crafted algorithms for learning each of these two sets of parameters. In addition, they lack finite sample guarantees. Since SGD and its variants remain the de facto algorithms for training neural networks because of their practical advantages, and inspired by the successes of these gradient-descent based algorithms in finding global minima in a variety of non-convex problems, we ask the following question:
 

\begin{question}
How do we design objective functions amenable to efficient optimization techniques, such as SGD, with provable learning guarantees for MoE?
\end{question}

In this paper, we address this question in a principled manner and propose two non-trivial non-convex loss functions $L_4(\cdot)$ and $L_{\log}(\cdot)$ to learn the regressors and the gating parameters respectively. In particular, our loss functions possess nice landscape properties such as local minima being global and the global minima corresponding to the ground truth parameters. We also show that gradient descent on our losses can recover the true parameters with \emph{global/random initializations}. To the best of our knowledge, ours is the first GD based approach with finite sample guarantees to learn the parameters of MoE. While our procedure to learn $\{a_i^\ast\}$ and $\{w_i^\ast\}$ separately and the technical assumptions are similar in spirit to \cite{makkuva2018breaking}, our loss function based approach with provable guarantees for SGD is significantly different from that of \cite{makkuva2018breaking}. We summarize our main contributions below:

\begin{itemize}
\item \textbf{MoE as a building block for GRU:} We provide the first connection that the well-known GRU models are composed of basic building blocks, known as MoE. This link provides important insights into theoretical understanding of GRUs and further highlights the importance of MoE.

\item \textbf{Optimization landscape design with desirable properties:} We design \emph{two non-trivial} loss functions $L_4(\cdot)$ and $L_{\log}(\cdot)$ to learn the regressors and the gating parameters of a MoE separately. We show that our loss functions have nice landscape properties and are amenable to simple local-search algorithms. In particular, we show that SGD on our novel loss functions recovers the parameters with \emph{global/random} initializations.
\item \textbf{First sample complexity results:} We also provide the first sample complexity results for MoE. We show that our algorithms can recover the true parameters with accuracy $\varepsilon$ and with high probability, when provided with samples \emph{polynomial} in the dimension $d$ and $1/\varepsilon$. 
\end{itemize}

\textbf{Related work.}
Linear dynamical systems can be thought of as the linear version of RNNs. There is a huge literature on the topic of learning these linear systems \cite{alaeddini2018linear,arora2018towards,dean2017sample, dean2018safely, marecek2018robust, oymak2018non, simchowitz2018learning,hardt2018gradient}. However these works are very specific to the linear setting and do not extend to non-linear RNNs. \cite{allen2018convergence} and \cite{allen2019can} are two recent works to provide first theoretical guarantees for learning RNNs with ReLU activation function. However, it is unclear how these techniques generalize to the gated architectures. In this paper, we focus on the learnability of MoE, which are the building blocks for these gated models. 

While there is a huge body of work on MoEs (see \cite{YWG12,MaEb14} for a detailed survey), the topic of learning MoE parameters is theoretically less understood with very few works on it. \cite{emmoeanalysis} is one of the early works that showed the local convergence of EM.  In a recent work, \cite{makkuva2018breaking} provided the first consistent algorithms for MoE in the population setting using a combination of spectral methods and EM algorithm. However, they do not provide any finite sample complexity bounds. In this work, we provide a unified approach using GD to learn the parameters with finite sample guarantees. To the best of our knowledge, we give the first gradient-descent based method with consistent learning guarantees, as well as the first finite-sample guarantee for any algorithm. The topic of designing the loss functions and analyzing their landscapes is a hot research topic in a wide variety of machine learning problems: neural networks \citep{identitymatters,kawa,reluanalysis,electron_proton,dhillon, landscapedesign,gao2018learning}, matrix completion \citep{bhoja}, community detection \citep{bandeira}, orthogonal tensor decomposition \citep{escapesaddle}. In this work, we present the first objective function design and the landscape analysis for MoE. 

\textbf{Notation.} We denote $\ell_2$-Euclidean norm by $\norm{\cdot}$. $[d] \define \{1,2,\ldots,d\}$. $\{e_i\}_{i=1}^d$ denotes the standard basis vectors in $\reals^d$. We denote matrices by capital letters like $A,W$, etc. For any two vectors $x,y \in \reals^d$, we denote their Hadamard product by $x \odot y$. $\sigma(\cdot)$ denotes the sigmoid function $\sigma(z)=1/(1+e^{-z}),z \in \reals$. For any $z=(z_1,\ldots,z_k) \in \reals^k$, $\mathrm{softmax}_i(z)=\exp(z_i)/(\sum_j \exp(z_j))$. $\calN(mu,\Sigma)$ denotes the Gaussian distribution with mean $\mu \in \reals^d$ and covariance $\Sigma \in \reals^{d\times d}$. Through out the paper, we interchangeably denote regressors as $\{a_i\}$ or $A$, and gating parameters as $\{w_i\}$ or $W$.

{\bf Overview.} The rest of the paper is organized as follows: In \prettyref{sec:connection}, we establish the precise mathematical connection between the well known GRU model and the MoE model. Building upon this correspondence, which highlights the importance of MoE, in \prettyref{sec:optim_landscape} we design two novel loss functions to learn the respective regressors and gating parameters of a MoE and present our theoretical guarantees. In \prettyref{sec:experiments}, we empirically validate that our proposed losses perform much better than the current approaches on a variety of settings. 


\section{GRU as a hierarchical MoE}
\label{sec:connection}
In this section, we show that the recurrent update equations for GRU can be obtained from that of MoE, described in \prettyref{eq:kmoe}. In particular, we show that GRU can be viewed as a hierarchical MoE with depth-2. To see this, we restrict to the setting of a $2$-MoE, \ie let $k=2$ and $(a_1^\ast, a_2^\ast)= (a_1, a_2)$, and $(w_1^\ast,w_2^\ast)=(w,0)$ in \prettyref{eq:kmoe}. Then we obtain that
\begin{align}
y = (1-z) ~ g(a_1^\top x) + z ~ g(a_2^\top x) + \xi, 
\label{eq:k2moe}
\end{align}
where $z \in \{0,1\} $ and $\prob{z=0|x} = \sigma(w^\top x)$. Since $\xi$ is a zero mean random variable independent of $x$, taking conditional expectation on both sides of \prettyref{eq:k2moe} yields that
\begin{align*}
y(x) & \define \Expect[y|x] \\
&= \sigma(w^\top x)g(a_1^\top x)+(1-\sigma(w^\top x)) g(a_2^\top x) \in \reals.
\end{align*}
Now letting the output $\by(x) \in \reals^m$ to be a vector and allowing for different gating parameters $\{w_i\}$ and regressors $\{(a_{1i},a_{2i})\}$ along each dimension $i =1,\ldots, m$, we obtain
\begin{align}
y(x) = (1-z(x)) \odot g(A_1 x) + z(x) \odot g(A_2 x),
\label{eq:vector_moe}
\end{align}
where $z(x)=(z_1(x),\ldots,z_m(x))^\top$ with $z_i(x)=\sigma(\bw_i^\top \bx)$, and $A_1,A_2 \in \reals^{m\times d}$ denote the matrix of regressors corresponding to first and second experts respectively. 

We now show that \prettyref{eq:vector_moe} is the basic equation behind the updates in GRU. Recall that in a GRU, given a time series $\{(x_t, y_t)\}_{t=1}^T$ of sequence length $T$, the goal is to produce a sequence of hidden states $\{h_t\}$ such that the output time series $\hat{y}_t = f(C h_t)$ is close to $\{y_t\}$ in some well-defined loss metric, where $f$ denotes the non-linear activation of the last layer. The equations governing the transition dynamics between $\{x_t \}$ and $\{h_t\}$ at any time $t \in [T]$ are given by \citep{gru}:
\begin{align*}
h_t = (1-z_t) \odot h_{t-1} +z_t \odot \tilde{h}_t,\\
\tilde{h}_t = g(U_h x_t + W_h (r_t \odot h_{t-1})),
\end{align*}
where $z_t$ and $r_t$ denote the update and reset gates, which are given by
\begin{align*}
 z_t &=\sigma(U_z x_t+W_z h_{t-1}) , \quad  r_t = \sigma(U_r x_t+W_r h_{t-1}),
\end{align*}
where the matrices $U$ and $W$ with appropriate subscripts are parameters to be learnt. While the gating activation function $\sigma$ is modeled as sigmoid for the ease of obtaining gradients while training, their intended purpose was to operate as binary valued gates taking values in $\{0,1\}$. Indeed, in a recent work \cite{li2018towards}, the authors show that binary valued gates enhance robustness with more interpretability and also give better performance compared to their continuous valued counterparts. In view of this, letting $\sigma$ to be the binary threshold function $\mathds{1}\{x \geq 0\}$, we obtain that
\begin{equation}
\begin{aligned}
h_t &= (1-z_t) \odot h_{t-1} + z_t \odot ( (1-r_t)\odot g(U_h x_t) \\
& \hspace{5em} + r_t \odot g(U_h x_t+ W_h h_{t-1})).
\end{aligned}
\label{eq:grumoe}
\end{equation}
Letting $\bx=(\bx_t,\bh_{t-1})$ and $\by(\bx)=\bh_t$ in \prettyref{eq:vector_moe} with second expert $g(A_2 x)$ replaced by a $2$-MoE, we can see from \prettyref{eq:grumoe} that GRU is a depth-$2$ hierarchical MoE. This is also illustrated in \prettyref{fig:gru_as_moe}.
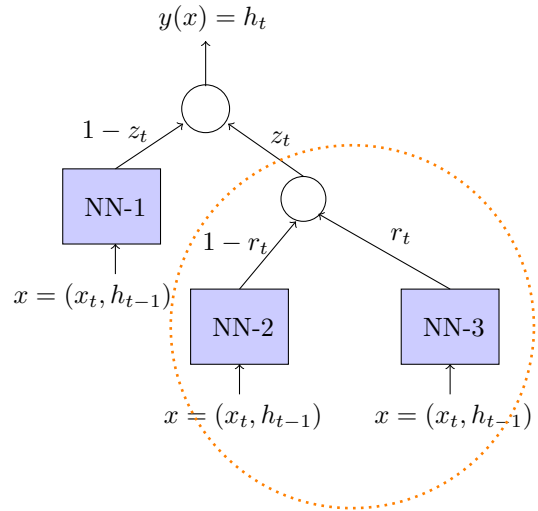
\begin{figure}[t]
\centering

\begin{tikzpicture}

\node[circle, draw=black, inner sep = 6.4 pt] at (-0.4,1.8) {};
\draw[->] (-0.4,2.1) -- (-0.4,2.7);
\node at (-0.3,3.0) {$y(x)=\bh_t$};

\filldraw[fill=blue!20!white, draw=black] (-2.3,0) -- (-1.0,0) -- (-1.0,1) -- (-2.3,1) -- cycle;
\node at (-1.6,0.5) {NN-1};

\node at (-1.9,-0.7) {$x=(\bx_t, \bh_{t-1})$};
\draw[->] (-1.6,-0.4)  -- (-1.6,0);

\node[circle, draw=black, inner sep=6pt] at (0.9,0.6) {};

\node at (-1.6 , 1.5) {$1-z_t$};
\draw[->] (-1.6,1)  -- (-0.7,1.6);

\node at (0.6  ,  1.4) {$z_t$};
\draw[->] (0.9,0.9)  -- (-0.1,1.6);

\filldraw[fill=blue!20!white, draw=black] (-0.6,-1.6) -- (0.7,-1.6) -- (0.7,-0.6) -- (-0.6,-0.6) -- cycle;
\node at (0.1,-1.1) {NN-2};

\node at (0.10,-2.3) {$x=(\bx_t, \bh_{t-1})$};
\draw[->] (0.05 , -2.0)  -- (0.05, -1.6);

\draw[->] (0.05,-0.6) -- (0.8,0.3) ;
\node at (0.0,0.0) {$1-r_t$};

\filldraw[fill=blue!20!white, draw=black] (2.2,-1.6) -- (3.5,-1.6) -- (3.5,-0.6) -- (2.2,-0.6) -- cycle;
\node at (2.9,-1.1) {NN-3};

\node at (2.90,-2.3) {$x=(\bx_t, \bh_{t-1})$};
\draw[->] (2.85 , -2.0)  -- (2.85, -1.6);

\draw[->] (2.85,-0.6) -- (1.1,0.4) ;
\node at (2.2,0.1) {$r_t$};

\node[draw,line width=1.0pt,color=orange,dotted,circle,inner ysep=15pt,fit={(-0.7,-1.1) (3.8,-1.1)}] {};

\end{tikzpicture}
\caption{GRU as a hierarchical $2$-MoE. The dotted circled portion indicates the canonical $2$-MoE in \prettyref{eq:k2moe}. NN-1, NN-2 and NN-3 denote specific input-output mappings obtained from \prettyref{eq:grumoe}. }
\label{fig:gru_as_moe}
\end{figure}

Note that in \prettyref{fig:gru_as_moe}, NN-1 models the mapping $(x_t,h_{t-1}) \mapsto h_{t-1}$, NN-2 represents $(x_t,h_{t-1}) \mapsto g(U_h x_{t})$, and NN-3 models $(x_t,h_{t-1}) \mapsto g(U_h x_t + W_h h_{t-1})$. Hence, this is slightly different from the traditional MoE setting in \prettyref{eq:kmoe} where the same activation $g(\cdot)$ is used for all the nodes. Nonetheless, we believe that studying this canonical model is a crucial first step which can shed important insights for a general setting.

\section{Optimization landscape design for MoE}
\label{sec:optim_landscape}
In the previous section, we presented the mathematical connection between the GRU and the MoE. In this section, we focus on the learnability of the MoE model and design two novel loss functions for learning the regressors and the gating parameters separately.

\subsection{Loss function for regressors: $L_4$}
\label{sec:L_4}

To motivate the need for loss function design in a MoE, first we take a moment to highlight the issues with the traditional approach of using the mean square loss $\ell_2$. If $(x,y)$ are generated according to the ground-truth MoE model in \prettyref{eq:kmoe}, $\ell_2(\cdot)$ computes the quadratic cost between the expected predictions $\hat{y}$ and the ground-truth $y$, \ie
\begin{align*}
\ell_2(\{a_i\}, \{w_i\}) = \Expect_{(x,y)}\|\hat{y}(x)-y\|^2,
\label{eq:l2_loss}
\end{align*} 
where $\hat{y}(x)= \sum_{i} \mathrm{softmax}_i(w_1^\top x, \ldots, w_{k-1}^\top x,0)~g(\ba_i^\top \bx)$ is the predicted output, and $\{a_i \},\{w_i\}$ denote the respective regressors and gating parameters. It is well-known that this mean square loss is prone to bad local minima as demonstrated empirically in the earliest work of \cite{JacJor} (we verify this in \prettyref{sec:experiments} too), which also emphasized the importance of the right objective function to learn the parameters. Note that the bad landscape of $\ell_2$ is not just unique to MoE, but also widely observed in the context of training neural network parameters \citep{livni2014computational}. In the one-hidden-layer NN setting, some recent works \citep{landscapedesign, gao2018learning} addressed this issue by designing new loss functions with good landscape properties so that standard algorithms like SGD can provably learn the parameters. However these methods do not generalize to the MoE setting since they crucially rely on the fact that the coefficients $z_i$ appearing in front of the activation terms $g(\inner{a_i^\ast}{x})$ in \prettyref{eq:kmoe}, which correspond to the linear layer weights in NN, are constant. Such an assumption does not hold in the context of MoEs because the gating probabilities depend on $x$ in a parametric way through the softmax function and hence introducing the coupling between $w_i^\ast$ and $a_i^\ast$ (a similar observation was noted in \cite{makkuva2018breaking} in the context of spectral methods).

%
%

In order to address the aforementioned issues, inspired by the works of \cite{landscapedesign} and \cite{gao2018learning}, we design a novel loss function $L_4(\cdot)$ to learn the regressors first. Our loss function depends on two distinct special transformations on both the input $x \in \reals^d$ and the output $y \in \reals$. For the output, we consider the following transformations:
\begin{align}
\calQ_4(y) \define y^4+\alpha y^3+ \beta y^2+ \gamma y, \hspace{0.4em} \calQ_2(y) \define y^2+ \delta y,
\end{align}
where the set of coefficients $(\alpha,\beta,\gamma,\delta)$ are dependent on the choice of non-linearity $g$ and noise variance $\sigma^2$.  These are obtained by solving a simple linear system (see \prettyref{app:non_linearclass}). For the special case $g=\Id$, which corresponds to linear activations, the Quartic transform is $\calQ_4(y)=y^4-6y^2(1+\sigma^2)+3+3\sigma^4-6\sigma^2$ and the Quadratic transform is $\calQ_2(y)=y^2-(1+\sigma^2)$. For the input $\bx$, we assume that $ x \sim \calN(0,I_d)$, and for any two fixed $u,v \in \reals^d$, we consider the projections of multivariate-Hermite polynomials \cite{grad1949note,holmquist1996d, scorebusiness} along these two vectors, \ie
\begin{align*}
t_3(\bu,\bx) &= \frac{(\bu^\top \bx)^2 - \norm{\bu}^2}{c'_{g,\sigma}}, \\
t_2(\bu, \bx) &= \frac{(\bu^\top \bx)^4 - 6 \norm{\bu}^2 (\bu^\top \bx)^2 + 3\norm{\bu}^4}{c_{g,\sigma}},\\
t_1(\bu, \bv ,\bx) &=((\bu^\top \bx)^2(\bv^\top \bx)^2 - \norm{\bu}^2(\bv^\top \bx)^2 \\
& \hspace{1.0em}- 4(\bu^\top \bx)(\bv^\top \bx)(\bu^\top \bv) - \norm{\bv}^2 (\bu^\top \bx)^2\\
& \hspace{1.0em} +\norm{\bu}^2 \norm{\bv}^2+ 2 (\bu^\top \bv)^2)/c_{g,\sigma},
\end{align*}
where $c_{g,\sigma}$ and $c'_{g,\sigma}$ are two non-zero constants depending on $g$ and $\sigma$. These transformations $(t_1, t_2, t_3)$ on the input $x$ and $(\calQ_4, \calQ_2)$ on the output $y$ can be viewed as extractors of higher order information from the data. The utility of these transformations is concretized in \prettyref{thm:landscape} through the loss function defined below.
Denoting the set of our regression parameters by the matrix $\bA^\top= [\ba_1|\ba_2|\ldots|\ba_k] \in \reals^{d \times k}$, we now define our objective function $L_4(\bA)$ as
\begin{align}
&L_4(\bA) \nonumber \\
&\define \sum_{\substack{i,j \in [k] \\i \neq j}} \Expect[\calQ_4(y) t_1(\ba_i,\ba_j,\bx)] - \mu \sum_{i \in [k]} \Expect[\calQ_4(y) t_2(\ba_i,\bx)] \nonumber \\  
& + \lambda \sum_{i \in [k]} \pth{\Expect[\calQ_2(y) t_3(\ba_i,\bx)]-1}^2
+\frac{\delta}{2} \norm{\bA}_F^2,
\label{eq:tensorloss}
\end{align}
where $\mu,\lambda, \delta>0$ are some positive regularization constants. Notice that $L_4$ is defined as an expectation of terms involving the data transformations: $\calQ_4, \calQ_2, t_1, t_2,$ and $t_3$. Hence its gradients can be readily computed from finite samples and is amenable to standard optimization methods such as SGD for learning the parameters. Moreover, the following theorem highlights that the landscape of $L_4$ does not have any spurious local minima.


\begin{theorem}[Landscape analysis for learning regressors]
\label{thm:landscape}
Under the mild technical assumptions of \cite{makkuva2018breaking}, the loss function $L_4$ does not have any spurious local minima. More concretely, let $\varepsilon>0$ be a given error tolerance. Then we can choose the regularization constants $\mu,\lambda$ and the parameters $\varepsilon,\tau$ such that if $A$ satisfies
\begin{align*}
\norm{\nabla L_4(\bA)}_2 \leq \varepsilon, \quad \nabla^2 L_4(\bA) \succcurlyeq  -\tau/2,
\end{align*}
then $(\bA^\dagger)^\top=   \bP \bD \Gamma \bA^\ast+ \bE$, where $\bD$ is a diagonal matrix with entries close to $1$, $\Gamma$ is a diagonal matrix with $\Gamma_{ii}=\sqrt{\Expect[p_i^\ast(\bx)]}$, $\bP$ is a permutation matrix and $\norm{\bE} \leq \varepsilon_0$. Hence every approximate local minimum is $\varepsilon$-close to the global minimum. 
\end{theorem}

\textbf{Intuitions behind the theorem and the special transforms:} While the transformations and the loss $L_4$ defined above may appear non-intuitive at first, the key observation is that $L_4$ can be viewed as a fourth-order polynomial loss in the parameter space, \ie
\begin{align}
&L_4(\bA) \nonumber \\
 & = \sum_{m \in [k]} \Expect[p_m^\ast(\bx)] \sum_{\substack{i \neq j \\ i,j \in [k]}} \inner{\ba_m^\ast}{\ba_i}^2 \inner{\ba_m^\ast}{\ba_j}^2 \nonumber \\
 & - \mu \sum_{m,i \in [k]} \Expect[p_m^\ast(\bx)] \inner{\ba_m^\ast}{\ba_i}^4   \label{eq:alternate} \\
& +\lambda \sum_{i \in [k]} (\sum_{m \in [k]} \Expect[p_m^\ast(\bx)] \inner{\ba^\ast_m}{\ba_i}^2-1)^2 +  \frac{\delta}{2} \norm{\bA}_F^2 \nonumber,
\end{align}
where $p_i^\ast$ refers to the softmax probability for the $i^{th}$ label with true gating parameters, \ie $p_i^\ast(x)=\mathrm{softmax}_i(\inner{w_1^\ast}{x},\ldots,\inner{w_{k-1}^\ast}{x},0)$. This alternate characterization of $L_4(\cdot)$ in \prettyref{eq:alternate} is the crucial step towards proving \prettyref{thm:landscape}. Hence these specially designed transformations on the data $(x,y)$ help us to achieve this objective. Given this viewpoint, we utilize tools from \cite{landscapedesign}, where a similar loss involving fourth-order polynomials were analyzed in the context of $1$-layer ReLU network, to prove the desired landscape properties for $L_4$. The full details behind the proof are provided in \prettyref{app:proofregressor}. Moreover, in \prettyref{sec:experiments} we empirically verify that the technical assumptions are only needed for the theoretical results and that our algorithms are robust to these assumptions, and work equally well even when we relax them.

In the finite sample regime, we replace the population expectations in \prettyref{eq:tensorloss} with sample average to obtain the empirical loss $\hat{L}$.  The following theorem establishes that $\hat{L}$ too inherits the same landscape properties of $L$ when provided enough samples.

\begin{theorem}[Finite sample landscape]
\label{thm:finitelandscape}
There exists a polynomial $\mathrm{poly}(d,1/\varepsilon)$ such that whenever $n \geq \mathrm{poly}(d,1/\varepsilon)$, $\hat{L}$ inherits the same landscape properties as that of $L$ established in \prettyref{thm:landscape} with high probability. Hence stochastic gradient descent on $\hat{L}$ converges to an approximate local minima which is also close to a global minimum in time polynomial in $d,1/\varepsilon$.
\end{theorem}
\begin{remark}\normalfont
Notice that the parameters $\{\ba_i\}$ learnt through SGD are some permutation of the true parameters $\ba_i^\ast$ upto sign flips. This sign ambiguity can be resolved using existing standard procedures such as Algorithm 1 in \cite{landscapedesign}. In the remainder of the paper, we assume that we know the regressors upto some error $\varepsilon_{\mathrm{reg}}>0$ in the following sense: $\max_{i \in [k]}\norm{a_i-a_i^\ast}=\sigma^2 \varepsilon_{\mathrm{reg}}$.
\end{remark}

\subsection{Loss function for gating parameters: $L_{\log}$}
\label{sec:gradgating}

In the previous section, we have established that we can learn the regressors $\ba_i^\ast$ upto small error using SGD on the loss function $L_4$. Now we are interested in answering the following question: Can we design a loss function amenable to efficient optimization algorithms such as SGD with recoverable guarantees to learn the gating parameters? 

In order to gain some intuition towards addressing this question, consider the simplified setting of $\sigma=0$ and $\bA=\bA^\ast$. In this setting, we can see from \prettyref{eq:kmoe} that the output $y$ equals one of the activation values $g(\inner{\ba_i^\ast}{\bx})$, for $i \in [k]$, with probability $1$. Since we already have access to the true parameters, \ie $A=\bA^\ast$, we can see that we can exactly recover the hidden latent variable $z$, which corresponds to the chosen hidden expert for each sample $(\bx,y)$. Thus the problem of learning the \textit{classifiers} $\bw_i^\ast,\ldots,\bw_{k-1}^\ast$ reduces to a multi-class classification problem with label $z$ for each input $\bx$ and hence can be efficiently solved by traditional methods such as logistic regression. It turns out that these observations can be formalized to deal with more general settings (where we only know the regressors approximately and the noise variance is not zero) and that the gradient descent on the log-likelihood loss achieves the same objective. Hence we use the negative log-likelihood function to learn the classifiers, \ie
\begin{align}
&L_{\log}(\bW,\bA) \nonumber \\
&\define  - \Expect_{(x,y)}[\log P_{y|\bx}] \label{eq:logloss} \\
&=-\Expect \log \pth{\sum_{i \in [k]} \frac{e^{\inner{\bw_i}{\bx}}}{\sum_{j \in [k]}e^{\inner{\bw_j}{\bx}}} \cdot \calN(y|g(\inner{\ba_i}{\bx}),\sigma^2)}, \nonumber
\end{align}
where $\bW^\top=\begin{bmatrix}
\bw_1| \bw_2 | \ldots| \bw_{k-1}
\end{bmatrix}$.
Note that the objective \prettyref{eq:logloss} in not convex in the gating parameters $W$ whenever $\sigma \neq 0$. We omit the input distribution $P_x$ from the above negative log-likelihood since it does not depend on any of the parameters. We now define the domain of the gating parameters $\Omega$ as
\begin{align*}
\bW \in \Omega \define \{\bW \in \reals^{(k-1)\times d}: \| w_i \|_2 \leq R, \forall i \in [k-1] \},
\end{align*}
for some fixed $R>0$. Without loss of generality, we assume that $\bw_k=0$. Since we know the regressors approximately from the previous stage, \ie $\bA \approx \bA^\ast$, we run gradient descent only for the classifier parameters keeping the regressors fixed, \ie
\begin{align*}
\bW_{t+1} =\Pi_{\Omega}( \bW_t - \alpha \nabla_{\bW}  L_{\log}(\bW_t,\bA)),
\end{align*}
where $\alpha>0$ is a suitably chosen learning-rate, $\Pi_{\Omega}(\bW)$ denotes the projection operator which maps each row of its input matrix onto the ball of radius $R$, and $t>0$ denotes the iteration step. In a more succinct way, we write
\begin{align*}
\bW_{t+1} &= G(\bW_t,\bA),\\
G(\bW,\bA) &\define \Pi_{\Omega}( \bW - \alpha \nabla_{\bW}  L_{\log}(\bW,\bA)).
\end{align*}
Note that $G(W,A)$ denotes the projected gradient descent operator on $W$ for fixed $A$. In the finite sample regime, we define our loss $L_{\log}^{(n)}(\bW,\bA)$ as the finite sample counterpart of \prettyref{eq:logloss} by taking empirical expectations.
Accordingly, we define the gradient operator $G_n(\bW,\bA)$ as
\begin{align*}
G_n(\bW,\bA) \define \Pi_{\Omega}( \bW - \alpha \nabla_{\bW}  L_{\log}^{(n)}(\bW,\bA)).
\end{align*}
In this paper, we analyze a sample-splitting version of the gradient descent, where given the number of samples $n$ and the iterations $T$, we first split the data into $T$ subsets of size $\lfloor n/T \rfloor$, and perform iterations on fresh batch of samples, \ie $\bW_{t+1}=G_{n/T}(\bW_t,\bA)$. We use the norm $\norm{W-W^\ast}= \max_{i \in [k-1]}\norm{w_i-w_i^\ast}_2$ for our theoretical results. The following theorem establishes the almost geometric convergence of the population-gradient iterates under some high SNR conditions. The following results are stated for $R=1$ for simplicity and also hold for any general $R>0$.
\begin{theorem}[GD convergence for classifiers]
\label{thm:twopopW}
Assume that $\max_{i \in [k]} \norm{a_i-a_i^\ast}_2 = \sigma^2 \varepsilon_{\mathrm{reg}} $. Then there exists two positive constants $\alpha_0$ and $\sigma_0$ such that for any step size $0<\alpha \leq \alpha_0$ and noise variance $\sigma^2<\sigma_0^2$, the population gradient descent iterates $\{\bW\}_{t \geq 0}$ converge almost geometrically to the true parameter $\bW^\ast$ for any randomly initialized $W_0 \in \Omega$, \ie
\begin{align*}
\norm{\bW_t-\bW^\ast} \leq \pth{\rho_\sigma}^t \norm{\bW_0 -\bW^\ast}+ \kappa \varepsilon_{\mathrm{reg}}\sum_{\tau=0}^{t-1}(\rho_\sigma)^\tau,
\end{align*}
where $(\rho_\sigma,\kappa) \in (0,1)\times (0,\infty)$ are dimension-independent constants depending on $g,k$ and $\sigma$ such that $\rho_\sigma=o_{\sigma}(1)$ and $\kappa = O_{k,\sigma}(1)$. 
\end{theorem}
\begin{proof}(Sketch)
For simplicity, let $\varepsilon_\mathrm{reg}=0$. Then we can show that $G(\bW^\ast,\bA^\ast)=\bW^\ast$ since $\nabla_{\bW}L_{\mathrm{\log}}(\bW=\bW^\ast,\bA^\ast)=0$. Then we capitalize on the fact that $G(\cdot,\bA^\ast)$ is strongly convex with minimizer at $\bW=\bW^\ast$ to show the geometric convergence rate. The more general case of $\varepsilon_\mathrm{reg}>0$ is handled through perturbation analysis. 
\end{proof}
We conclude our theoretical discussion on MoE by providing the following finite sample complexity guarantees for learning the classifiers using the gradient descent in the following theorem, which can be viewed as a finite sample version of \prettyref{thm:twopopW}.

\begin{theorem}[Finite sample complexity and convergence rates for GD]
\label{thm:twosampW}
In addition to the assumptions of \prettyref{thm:twopopW}, assume that the sample size $n$ is lower bounded as $n \geq c_1 T d \log(\frac T \delta)$. Then the sample-gradient iterates $\{ \bW^t \}_{t=1}^T$ based on $n/T$ samples per iteration satisfy the bound
\begin{align*}
\norm{\bW^t-\bW^\ast} &\leq (\rho_\sigma)^t \norm{\bW_0-\bW^\ast}\\
&\hspace{1em}+\frac{1}{1-\rho_\sigma}\pth{\kappa \varepsilon_{\mathrm{reg}}  + c_2 \sqrt{\frac{d T \log(T k/\delta)}{n}}}
\end{align*}
with probability at least $1-\delta$.
\end{theorem}

%

\begin{figure*}[t]
    \centering
    \begin{subfigure}[b]{0.32\textwidth}
        \includegraphics[width=\textwidth]	{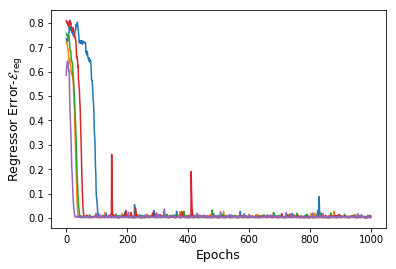}
		\caption{} \label{fig:1a}
	\end{subfigure}
    \begin{subfigure}[b]{0.32\textwidth}
        \includegraphics[width=\textwidth]	{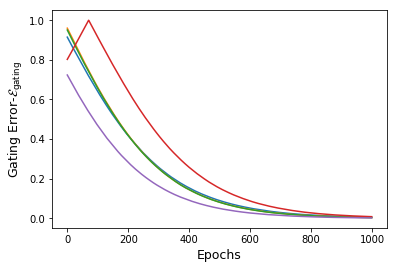}
		\caption{} \label{fig:1b}
	\end{subfigure}
	    \begin{subfigure}[b]{0.32\textwidth}
        \includegraphics[width=\textwidth]	{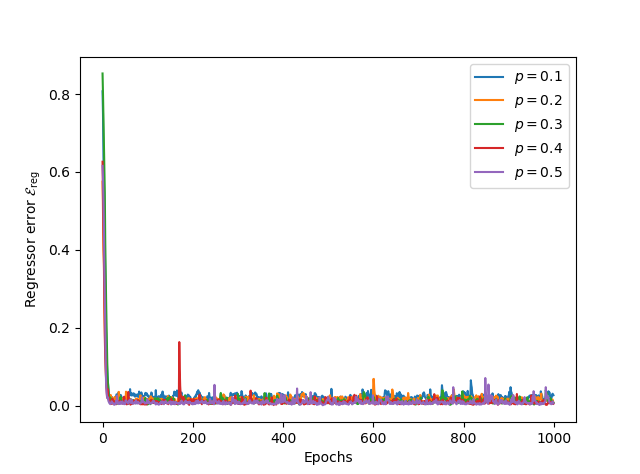}
		\caption{} \label{fig:1c}
	\end{subfigure}
\caption{(a), (b): Robustness to parameter orthogonality: Plots show performance over $5$ different trials for our losses $L_4$ and $L_{\mathrm{log}}$ respectively. (c) Robustness to Gaussianity of input: Performance over various mixing probabilities $p$. }
\label{fig:all}
\end{figure*}

\section{Experiments}
\label{sec:experiments}

In this section, we empirically validate the fact that running SGD on our novel loss functions $L_4$ and $L_\mathrm{\log}$ achieves superior performance compared to the existing approaches. Moreover, we empirically show that our algorithms are robust to the technical assumptions made in \prettyref{thm:landscape} and that they achieve equally good results even when the assumptions are relaxed.

{\bf Data generation.} For our experiments, we choose $d=10$, $k \in \{2,3\}$, $\ba_i^\ast=\be_i$ for $i \in [k]$ and $\bw_i^\ast=\be_{k+i}$ for $i \in [k-1]$, and $g=\Id$. We generate the data $\{(\bx_i,y_i)_{i=1}^n \}$ according to \prettyref{eq:kmoe} and using these ground-truth parameters. We chose $\sigma=0.05$ for all of our experiments.

{\bf Error metric.} If $\bA \in \reals^{k \times d}$ denotes the matrix of regressors where each row is of norm $1$, we use the error metric $\calE_{\mathrm{reg}}$ to gauge the closeness of $\bA$ to the ground-truth $\bA^\ast$:
\begin{align*}
\calE_{\mathrm{reg}} \define 1- \max_{\pi \in S_k} \min_{i \in [k]} |\inner{\ba_i}{\ba^\ast_{\pi(i)}}|,
\end{align*}
where $S_k$ denotes the set of all permutations on $[k]$. Note that $\calE_{\mathrm{reg}} \leq \varepsilon$ if and only if the learnt regressors have a minimum correlation of $1-\varepsilon$ with the ground-truth parameters, upto a permutation. The error metric $\calE_{\mathrm{gating}}$ is defined similarly.

{\bf Results.} In \prettyref{fig:regresslearn}, we choose $k=3$ and compare the performance of our algorithm against existing approaches. In particular, we consider three methods: 1) EM algorithm, 2) SGD on the the classical $\ell_2$-loss from \prettyref{eq:l2_loss}, and 3) SGD on our losses $L_4$ and $L_{\mathrm{log}}$. For all the methods, we ran $5$ independent trials and plotted the mean error. \prettyref{fig:all_regressor_error} highlights the fact that minimizing our loss function $L_4$ by SGD recovers the ground-truth regressors, whereas SGD on $\ell_2$-loss as well as EM get stuck in local optima. For learning the gating parameters $W$ using our approach, we first fix the regressors $A$ at the values learnt using $L_4$, \ie $A=\hat{A}$, where $\hat{A}$ is the converged solution for $L_4$. For $\ell_2$ and the EM algorithm, the gating parameters $W$ are learnt jointly with regressors $A$. \prettyref{fig:all_classifier_error} illustrates the phenomenon that our loss $L_{\mathrm{log}}$ for learning the gating parameters performs considerably better than the standard approaches, as indicated in significant gaps between the respective error values. Finally, in \prettyref{fig:our_loss_multiple_runs} we plot the regressor error for $L_4$ over $5$ random initializations. We can see that we recover the ground truth parameters in all the trials, thus empirically corroborating our technical results in \prettyref{sec:optim_landscape}.



\subsection{Robustness to technical assumptions}
In this section, we verify numerically the fact that our algorithms work equally well in the absence of technical assumptions made in \prettyref{sec:optim_landscape}.

\textbf{Relaxing orthogonality in \prettyref{thm:landscape}.} A key assumption in proving \prettyref{thm:landscape}, adapted from \cite{makkuva2018breaking}, is that the set of regressors $\{a_i^\ast\}$ and set of gating parameters $\{w_i^\ast\}$ are orthogonal to each other. While this assumption is needed for the technical proofs, we now empirically verify that our conclusions still hold when we relax this. For this experiment, we choose $k=2$ and let $(a_1^\ast, a_2^\ast)=(e_1, e_2)$. For the gating parameter $w^\ast \define w_1^\ast$, we randomly generate it from uniform distribution on the $d$-dimensional unit sphere. In \prettyref{fig:1a} and \prettyref{fig:1b}, we plotted the individual parameter estimation error for $5$ different runs for both of our losses $L_4$ and $L_{\mathrm{log}}$ for learning the regressors and the gating parameter respectively. We can see that our algorithms are still able to learn the true parameters even when the orthogonality assumption is relaxed. 

\textbf{Relaxing Gaussianity of the input.} To demonstrate the robustness of our approach to the assumption that the input $x$ is standard Gaussian, \ie $x \sim \calN(0,I_d)$, we generated $x$ according to a mixture of two symmetric Gaussians each with identity covariances, \ie $x \sim p \calN(\mu, I_d)+(1-p) \calN(-\mu,I_d)$, where  $p \in [0,1]$ is the mixing probability and $\mu \in \reals^d$ is a fixed but randomly chosen vector. For various mixing proportions $p \in \{0.1, 0.2, 0.3, 0.4, 0.5\}$, we ran SGD on our loss $L_4$ to learn the regressors.  \prettyref{fig:1c} highlights that we learn these ground truth parameters in all the settings. 

Finally we note that in all our experiments, the loss $L_4$ seems to require a  larger batch size ($1024$) for its gradient estimation while running SGD. However, with smaller batch sizes such as $128$ we are still able to achieve similar performance but with more variance. (see \prettyref{app:additionalexp}).

%



\section{Discussion}
\label{sec:disc}

In this paper we established the first mathematical connection between two popular gated neural networks: GRU and MoE. Inspired by this connection and the success of SGD based algorithms in finding global minima in a variety of non-convex problems in deep learning, we provided the first gradient descent based approach for learning the parameters in a MoE. While the canoncial MoE does not involve any time series, extension of our methods for the recurrent setting is an important future direction. Similarly, extensions to deep MoE comprised of multiple gated as well as non-gated layers is also a fruitful direction of further research. We believe that the theme of using different loss functions for distinct parameters in NN models can potentially enlighten some new theoretical insights as well as practical methodologies for complex neural models.

\subsubsection*{Acknowledgements}
We would like to thank the anonymous reviewers for their suggestions. This work is supported by NSF grants 1927712 and 1929955.

\bibliographystyle{apalike}
\bibliography{references}

\appendix
\onecolumn
\section{Connection between $k$-MoE and other popular models}
\label{app:connection}

\paragraph{Relation to other mixture models.}
Notice that if let $\bw_i^\ast=0$ in \prettyref{eq:kmoe} for all $i \in [k]$, we recover the well-known uniform mixtures of \textit{generalized linear models (GLMs)}. Similarly, allowing for bias parameters in \prettyref{eq:kmoe}, we can recover the generic mixtures of GLMs. Moreover, if we let $g$ to be the linear function, we get the popular \textit{mixtures of linear regressions} model. These observations highlight that MoE models are a far more stricter generalization of mixtures of GLMs since they allow the mixing probability $p_i^\ast(\bx)$ to depend on each input $\bx$ in a parametric way. This makes the learning of the parameters far more challenging since the gating and expert parameters are inherently coupled.

\paragraph{Relation to feed-forward neural networks.} Note that if we let $\bw_i^\ast=0$ and allow for bias parameters in the soft-max probabilities in \prettyref{eq:kmoe}, taking conditional expectation on both sides yields
\begin{align}
\hat{y}(\bx) \define \Expect[y|\bx] = \sum_{i \in [k]} w_i^\ast g(\inner{\ba_i^\ast}{\bx}), \quad \sum_i w_i^\ast =1, w_i^\ast \in [0,1].
\label{eq:nn}
\end{align}
Thus the mapping $\bx \mapsto \hat{y}(\bx)$ is exactly the same as that of a $1$-hidden -layer neural network with activation function $g$ if we restrict the output layer to positive weights. Thus $k$-MoE can also be viewed as a probabilistic model for \textit{gated feed-forward networks}.

\section{Valid class of non-linearities}
\label{app:non_linearclass}
We slightly modify the class of non-linearities from \cite{makkuva2018breaking} for our theoretical results. The only key modification is that we use a fourth-order derivative based conditions, as opposed to third-order derivatives used in the above work. Following their notation, let $Z \sim \calN(0,1)$ and $Y|Z \sim \calN(g(Z),\sigma^2)$, where $g: \reals \to \reals$. For $(\alpha, \beta, \gamma,\delta ) \in \reals^4$, define
\begin{align*}
\calQ_4(y) \triangleq Y^4+\alpha Y^3+\beta Y^2+ \gamma Y,
\end{align*}
where
\begin{align*}
\calS_4(Z) \define \Expect[ \calQ_4(y)|Z]=g(Z)^4+6g(Z)^2 \sigma^2+\sigma^4+\alpha(g(Z)^3+3g(Z)\sigma^2) + \beta(g(Z)^2+\sigma^2)+\gamma g(Z).
\end{align*}
Similarly, define
\begin{align*}
\calQ_2(y) &\triangleq Y^2+\delta Y, \quad  \calS_2(Z)= \Expect[ \calQ_2(y)|Z]=g(Z)^2+ \delta g(Z)+\sigma^2.
\end{align*}
\begin{condition}\normalfont
\label{cond:cond1}$\Expect[\calS_4'(Z)]=\Expect[\calS_4''(Z)]=\Expect[\calS_4'''(Z)]=0$ and $\Expect[\calS_4''''(Z)] \neq 0$. Or equivalently, in view of Stein's lemma \cite{Ste72},
\begin{align*}
\Expect[\calS_4(Z) Z]=\Expect[\calS_4(Z)(Z^2-1)]=\Expect[\calS_4(Z)(Z^3-3Z)]=0, \text { and } \Expect[\calS_4(Z)(Z^4-6Z^2+3)] \neq 0.
\end{align*}
\end{condition}
\begin{condition}\normalfont
\label{cond:cond2}$\Expect[\calS_2'(Z)]=0$ and $\Expect[\calS_2''(Z)] \neq 0$. Or equivalently,
\begin{align*}
\Expect[\calS_2(Z) Z]=0 \text{ and } \Expect[\calS_2(Z) (Z^2-1)] \neq 0.
\end{align*}
\end{condition}

\begin{definition}
We say that the non-linearity $g$ is $(\alpha,\beta,\gamma,\delta)-\valid $ if there exists a tuple $(\alpha,\beta,\gamma, \delta) \in \reals^4$ such that both  \prettyref{cond:cond1} and \prettyref{cond:cond2} are satisfied.
\end{definition}


While these conditions might seem restrictive at first, all the widely used non-linearities such as $\Id$, ReLU, leaky-ReLU, sigmoid, etc. belong to this. For some of these non-linear activations, we provide the pre-computed transformations below:

\begin{example}
If $g=\Id$, then $\calS_3(y)=y^4-6y^2(1+\sigma^2)$ and $\calQ_2(y)=y^2$.
\end{example}
\begin{example}
If $g=$ ReLU, \ie $g(z)=\mathrm{max}\{0,z\}$, we have that for any $p,q \in \naturals$,
\begin{align*}
\Expect[g(Z)^p Z^q] = \int_0^\infty z^{p+q} \pth{\frac{1}{\sqrt{2\pi}} e^{-z^2/2}} dz =\frac{1}{2}\Expect[|Z|^{p+q}]=\frac{(p+q-1)!!}{2}\begin{cases}
 \sqrt{\frac{2}{\pi}} & \text{if } p+q \text{ is odd} \\ 1 & \text{if } p+q \text{ is even}
\end{cases}.
\end{align*}
\end{example}
Substituting these moments in the linear set of equations $\Expect[\calS_4(Z) Z]=\Expect[\calS_4(Z)(Z^2-1)]=\Expect[\calS_4(Z)(Z^3-3Z)]=0$, we obtain
\begin{align*}
\begin{bmatrix}
1.5+1.5\sigma^2 & \sqrt{\frac{2}{\pi}}+\sigma^2 & 0.5 \\
3\sqrt{\frac{2}{\pi}}(1+\sigma^2/2) & 1+\sigma^2 & \frac{1}{2}\sqrt{\frac{2}{\pi}} \\
3 & \sqrt{\frac{2}{\pi}}+\sigma^2 & 0 
\end{bmatrix}\begin{bmatrix}
\alpha \\ \beta \\ \gamma
\end{bmatrix}= - \begin{bmatrix}
\sqrt{\frac{2}{\pi}}(4+6\sigma^2)) \\ 6+6\sigma^2 \\ \sqrt{\frac{2}{\pi}}(12+6\sigma^2)
\end{bmatrix}.
\end{align*}
Solving for $(\alpha,\beta,\gamma)$ will yield $\calS_4(Z)$. Finally, we have that $\delta=-2\sqrt{\frac{2}{\pi}}$.

\section{Proofs of \prettyref{sec:L_4}}
\label{app:proofregressor}

\begin{remark}\normalfont
To choose the parameters in \prettyref{thm:landscape}, we follow the parameter choices from \cite{landscapedesign}. Let $c$ be a sufficiently small universal constant (\eg $c=0.01$). Assume $\mu \leq c/\kappa^\ast$, and $\lambda \geq 1/(ca_{\min}^\ast)$. Let $\tau_0=c \min \sth{\mu/(\kappa d a_{\max}^\ast), \lambda}\sigma_{\min}(\bM)$. Let $\delta \leq \min \sth{\frac{c \varepsilon_0}{a_{\max}^\ast \cdot m \sqrt{d} \kappa^{1/2}(\bM)}, \tau_0/2 }$ and $\varepsilon=\min \sth{\lambda \sigma_{\min}(\bM)^{1/2}, c\delta/\sqrt{\norm{\bM}}, c\varepsilon_0\delta\sigma_{\min}(\bM) }$.
\end{remark}
%

For any $k \times d$ matrix $\bA$, let $\bA^\dagger$ be its pseudo inverse such that $\bA \bA^\dagger= \bI_{k \times k}$ and $\bA^\dagger \bA$ is the projection matrix to the row span of $\bA$. Let $\alpha_i^\ast \define \Expect[p_i^\ast(\bx)], a_i^\ast=\frac{1}{\alpha_i^\ast}$ and $\kappa^\ast=\frac{\alpha_{\max}^\ast}{\alpha_{\min}^\ast}$. Let $\bM=\sum_{i \in [k]} \alpha_i^\ast \ba_i^\ast (\ba_i^\ast)^\top$, $\kappa(\bM)= \frac{\norm{\bM}}{\sigma_{\min}(\bM)}$.

For the sake of clarity, we now formally state our main assumptions, adapted from \cite{makkuva2018breaking}:

\begin{enumerate}
\item $\bx$ follows a standard Gaussian distribution, \ie $\bx \sim \calN(0,\bI_d)$.
\item $\norm{\ba_i^\ast}=1$ for all $i \in [k]$ and $\norm{\bw_i^\ast} \leq R$ for all $i \in [k-1]$.
\item The regressors $\ba_1^\ast,\ldots,\ba_k^\ast$ are linearly independent and the classifiers $\{\bw_i^\ast\}_{i \in [k-1]}$ are orthogonal to the span $\calS=\mathrm{span}\sth{\ba_1^\ast,\ldots,\ba_k^\ast}$, and $2k-1 < d$.
\item The non-linearity $g: \reals \to \reals$ is $(\alpha,\beta,\gamma,\delta)-\valid$, which we define in \prettyref{app:non_linearclass}.
\end{enumerate}

Note that while the first three assumptions are same as that of \cite{makkuva2018breaking}, the fourth assumption is slightly different from theirs. Under this assumptions, we first give an alternative characterization of $L_4(\cdot)$ in the following theorem which would be crucial for the proof of \prettyref{thm:landscape}.
\begin{theorem}
\label{thm:losslhs}
The function $L(\cdot)$ defined in \prettyref{eq:tensorloss} satisfies that
\begin{align*}
L_4(\bA) = \sum_{m \in [k]} \Expect[p_m^\ast(\bx)] \sum_{\substack{i \neq j \\ i,j \in [k]}} \inner{\ba_m^\ast}{\ba_i}^2 \inner{\ba_m^\ast}{\ba_j}^2 - \mu \sum_{m,i \in [k]} \Expect[p_m^\ast(\bx)] \inner{\ba_m^\ast}{\ba_i}^4 \\
+ \lambda \sum_{i \in [k]} (\sum_{m \in [k]} \Expect[p_m^\ast(\bx)] \inner{\ba^\ast_m}{\ba_i}^2-1)^2  +  \frac{\delta}{2} \norm{\bA}_F^2 
\end{align*}
\end{theorem}

\subsection{Proof of \prettyref{thm:losslhs}}
\begin{proof}
For the proof of \prettyref{thm:losslhs}, we use the notion of score functions defined as \cite{scorebusiness}:
\begin{align}
\calS_m(\bx) \define (-1)^m \frac{\nabla_{\bx}^{(m)}f(\bx)}{f(\bx)}, \quad f \text{ is the pdf of } \bx.
\end{align}
In this paper we focus on $m=2,4$. When $\bx \sim \calN(0,\bI_d)$, we know that $\calS_2(\bx) = \bx \otimes \bx - \bI$ and
\begin{align*}
\calS_4(\bx) = \bx^{\otimes 4}-\sum_{i \in [d]} \sym\pth{\bx \otimes \be_i \otimes \be_i \otimes \bx }+\sum_{i,j} \sym \pth{\be_i \otimes \be_i \otimes \be_j \otimes \be_j}.
\end{align*}
The score transformations $\calS_4(\bx)$ and $\calS_2(\bx)$ can be viewed as multi-variate polynomials in $\bx$ of degrees $4$ and $2$ respectively. For the output $y$, recall the transforms $\calQ_4(y)$ and $\calQ_2(y)$ defined in \prettyref{sec:L_4}. The following lemma shows that one can construct a fourth-order super symmetric tensor using these special transforms.

\begin{lemma}[Super symmetric tensor construction]
\label{lmm:fourthtensor}
Let $(\bx,y)$ be generated according to \prettyref{eq:kmoe} and Assumptions $(1)$-$(4)$ hold. Then
\begin{align*}
\calT_4 \define \Expect[\calQ_4(y) \cdot \calS_4(\bx)]&= c_{g,\sigma} \sum_{i \in [k]}\Expect[p_i^\ast(\bx)] \cdot \ba_i^\ast \otimes \ba_i^\ast \otimes \ba_i^\ast \otimes \ba_i^\ast,\\
\calT_2 \define \Expect[\calQ_2(y) \cdot \calS_2(\bx)] &=c'_{g,\sigma} \sum_{i \in [k]}\Expect[p_i^\ast(\bx)] \cdot \ba_i^\ast \otimes \ba_i^\ast,
\end{align*}
where $p_i^\ast(x)=\prob{z_i=1|x}$, $c_{g,\sigma}$ and $c'_{g,\sigma}$ are two non-zero constants depending on $g$ and $\sigma$.
\end{lemma}

Now the proof of the theorem immediately follows from \prettyref{lmm:fourthtensor}. Recall from \prettyref{eq:tensorloss} that
\begin{align*}
L_4(\bA) \define \sum_{\substack{i,j \in [k] \\i \neq j}} \Expect[\calQ_4(y) t_1(\ba_i,\ba_j,\bx)] - \mu \sum_{i \in [k]} \Expect[\calQ_4(y) t_2(\ba_i,\bx)] + \lambda \sum_{i \in [k]} \pth{\Expect[\calQ_2(y) t_3(\ba_i,\bx)]-1}^2\\
+\frac{\delta}{2} \norm{\bA}_F^2.
\end{align*}
Fix $i,j \in [k]$. Notice that we have $t_1(\ba_i,\ba_j,\bx) = \calS_4(\bx)(\ba_i,\ba_i,\ba_j,\ba_j)/c_{g,\sigma}$. Hence we obtain
\begin{align*}
\Expect[\calQ_4(y) t_1(\ba_i,\ba_j,\bx)] &=\frac{1}{c_{g,\sigma}}\Expect[\calQ_4(y)\cdot \calS_4(\bx)](\ba_i,\ba_i,\ba_j,\ba_j)\\
&=\pth{\sum_{m \in [k]}\Expect[p_m^\ast(\bx)](\ba_m^\ast)^{\otimes 4}}(\ba_i,\ba_i,\ba_j,\ba_j) \\
&=\sum_{m \in [k]}\Expect[p_m^\ast(\bx)] \inner{\ba_m^\ast}{\ba_i}^2 \inner{\ba_m^\ast}{\ba_j}^2.
\end{align*}
The simplification for the remaining terms is similar and follows directly from definitions of $t_2(\cdot,\bx)$ and $t_3(\cdot,\bx)$.
\end{proof}

\subsection{Proof of \prettyref{thm:landscape}}
\begin{proof}
The proof is an immediate consequence of \prettyref{thm:losslhs} and Theorem C.5 of \cite{landscapedesign}.
\end{proof}

\subsection{Proof of \prettyref{thm:finitelandscape}}
\begin{proof}
Note that our loss function $L_4(\bA)$ can be written as $\Expect[\ell(\bx,y,\bA)]$ where $\ell$ is at most a fourth degree polynomial in $\bx$, $y$ and $\bA$. Hence our finite sample guarantees directly follow from \prettyref{thm:landscape} and Theorem E.1 of \cite{landscapedesign}.
\end{proof}

\subsection{Proof of \prettyref{lmm:fourthtensor}}
\begin{proof}

The proof of this lemma essentially follows the same arguments as that of \cite[Theorem 1]{makkuva2018breaking}, where we replace $(\calS_3(x),\calS_2(x),\calP_3(y),\calP_2(y))$ with $(\calS_4(x),\calS_2(x),\calQ_4(y),\calP_2(y))$ respectively and letting $\calT_3$ defined there with our $\calT_4$ defined above.

\end{proof}

\section{Proofs of \prettyref{sec:gradgating}}
For the convergence analysis of SGD on $L_{\mathrm{\log}}$, we use techniques from \cite{bala17} and \cite{makkuva2018breaking}. In particular, we adapt \cite[Lemma 3]{makkuva2018breaking} and \cite[Lemma 4]{makkuva2018breaking} to our setting through \prettyref{lmm:contraction} and \prettyref{lmm:robustness}, which are central to the proof of \prettyref{thm:twopopW} and \prettyref{thm:twosampW}. We now sate our lemmas.

\begin{lemma}
\label{lmm:contraction}
Under the assumptions of \prettyref{thm:twopopW}, it holds that
\begin{align*}
\|G(\bW,\bA^\ast)-\bW^\ast_i \| \leq \rho_\sigma \|\bW-\bW^\ast\|.
\end{align*}
In addition, $\bW=\bW^\ast$ is a fixed point for $G(\bW,\bA^\ast)$.
\end{lemma}

\begin{lemma}
\label{lmm:robustness}
Let the matrix of regressors $\bA$ be such that $\max_{i \in [k]}\| \bA_i^\top - (\bA^\ast_i)^\top \|_2 = \sigma^2 \varepsilon $. Then for any $\bW \in \Omega$, we have that
\begin{align*}
\|G(\bW,\bA)-G(\bW,\bA^\ast) \| \leq \kappa \varepsilon,
\end{align*}
where $\kappa$ is a constant depending on $g,k$ and $\sigma$. In particular, $\kappa \leq  (k-1)\frac{\sqrt{6(2+\sigma^2)}}{2}$ for $g=$linear, sigmoid and ReLU.
\end{lemma}

\begin{lemma}[Deviation of finite sample gradient operator]
\label{lmm:deviation}
For some universal constant $c_1$, let the number of samples $n$ be such that $n \geq c_1 d \log(1/\delta)$. Then for any fixed set of regressors $\bA \in \reals^{k \times d}$, and a fixed $\bW \in \Omega$, the bound
\begin{align*}
\| G_{n}(\bW,\bA)-G(\bW,\bA) \| \leq \varepsilon_G(n,\delta) \define c_2 \sqrt{\frac{d\log(k/\delta)}{n}}
\end{align*}
holds with probability at least $1-\delta$.
\end{lemma}

\subsection{Proof of \prettyref{thm:twopopW}}
\begin{proof}
The proof directly follows from \prettyref{lmm:contraction} and \prettyref{lmm:robustness}.
\end{proof}

\subsection{Proof of \prettyref{thm:twosampW}}
\begin{proof}
Let the set of regressors $\bA$ be such that $\max_{i \in [k]}\| \bA_i^\top - (\bA^\ast_i)^\top \|_2 = \sigma^2 \varepsilon_1 $. Fix $\bA$. For any iteration $t \in [T]$, from \prettyref{lmm:deviation} we have the bound
\begin{align}
\| G_{n/T}(\bW_t,\bA)-G(\bW_t,\bA) \| \leq \varepsilon_G(n/T,\delta/T)
\label{eq:samplesplitbound}
\end{align}
with probability at least $1-\delta/T$. Using an union bound argument, \prettyref{eq:samplesplitbound} holds with probability at least $1-\delta$ for all $t \in [T]$. Now we show that the following bound holds:
\begin{align}
\norm{\bW_{t+1}-\bW^\ast} \leq \rho_\sigma \norm{\bW_{t}-\bW^\ast} + \kappa \varepsilon_1+\varepsilon_G(n/T,\delta/T), \quad \text{ for each } t \in \{0,\ldots,T-1\}.
\end{align}
Indeed, for any $t \in \{0,\ldots,T-1\}$, we have that
\begin{align*}
\norm{\bW_{t+1}-\bW^\ast} &= \norm{G_{n/T}(\bW_t,\bA)-\bW^\ast} \\
&\leq \norm{G_{n/T}(\bW_t,\bA)-G(\bW_t,\bA)}+\norm{G(\bW_t,\bA)-G(\bW_t,\bA^\ast)}+\norm{G(\bW_t,\bA^\ast)-\bW^\ast} \\
& \leq \varepsilon_G(n/T,\delta/T) + \kappa \varepsilon_1 + \rho_\sigma \norm{\bW_t-\bW^\ast},
\end{align*}
where we used in \prettyref{lmm:contraction}, \prettyref{lmm:robustness} and \prettyref{lmm:deviation} in the last inequality to bound each of the terms. From \prettyref{eq:samplesplitbound}, we obtain that
\begin{align*}
\norm{\bW_t-\bW^\ast} &\leq \rho_\sigma \norm{\bW_{t-1}-\bW^\ast}+\kappa \varepsilon_1+\varepsilon_G(n/T,\delta/T) \\
& \leq \rho_\sigma^2 \norm{\bW_{t-2}-\bW^\ast}+(1+\rho_\sigma)\pth{\kappa \varepsilon_1+\varepsilon_G(n/T,\delta/T)}\\
& \leq \rho_\sigma^t \norm{\bW_{0}-\bW^\ast}+\pth{\sum_{s=0}^{t-1}\rho_\sigma^s}\pth{\kappa \varepsilon_1+\varepsilon_G(n/T,\delta/T)} \\
& \leq \rho_\sigma^t \norm{\bW_{0}-\bW^\ast}+\pth{\frac{1}{1-\rho_\sigma}}\pth{\kappa \varepsilon_1+\varepsilon_G(n/T,\delta/T)}. 
\end{align*}
\end{proof}

\subsection{Proof of \prettyref{lmm:contraction}}
\begin{proof}
Recall that the loss function for the population setting, $L_{\log}(\bW,\bA)$, is given by
\begin{align*}
L_{\log}(\bW,\bA) = -\Expect \log \pth{\sum_{i \in [k]} \frac{e^{\inner{\bw_i}{\bx}}}{\sum_{j \in [k]}e^{\inner{\bw_j}{\bx}}} \cdot \calN(y|g(\inner{\ba_i}{\bx}),\sigma^2)}=-\Expect \log \pth{\sum_{i \in [k]}p_i(\bx)N_i},
\end{align*}
where $p_i(\bx)\define \frac{e^{\inner{\bw_i}{\bx}}}{\sum_{j \in [k]}e^{\inner{\bw_j}{\bx}}}$ and $N_i \define \calN(y|g(\inner{\ba_i}{\bx}),\sigma^2)$. Hence for any $i \in [k-1]$, we have
\begin{align*}
\nabla_{\bw_i}L_{\log}(\bW,\bA) = -\Expect \pth{\frac{\nabla_{\bw_i}p_i(\bx)N_i +\sum_{j \neq i, j \in [k]} \nabla_{\bw_i}p_j(\bx)N_j }{\sum_{i \in [k]}p_i(\bx)N_i}}.
\end{align*}
Moreover,
\begin{align*}
\nabla_{\bw_i} p_j(\bx) = \begin{cases}
p_i(\bx)(1-p_i(\bx)) \bx, & j=i \\
-p_i(\bx)p_j(\bx) \bx, & j \neq i \\
\end{cases}.
\end{align*}
Hence we obtain that
\begin{align}
\nabla_{\bw_i}L_{\log}(\bW,\bA) = -\Expect \qth{\frac{p_i(\bx)N_i}{\sum_{i \in [k]}p_i(\bx)N_i} - p_i(\bx)}.
\label{eq:gradcomputation}
\end{align}
Notice that if $z \in [k]$ denotes the latent variable corresponding to which expert is chosen, we have that the posterior probability of choosing the $i$th expert is given by
$$\prob{z=i|\bx,y} =\frac{p_i(\bx)N_i}{\sum_{i \in [k]}p_i(\bx)N_i},$$
whereas,
\begin{align*}
\prob{z=i|\bx}=p_i(\bx).
\end{align*}
Hence, when $\bA=\bA^\ast$ and $\bW=\bW^\ast$, we get that
\begin{align*}
\nabla_{\bw_i^\ast}L_{\log}(\bW^\ast,\bA^\ast)= -\Expect[\prob{z=i|\bx,y}-\prob{z=i|\bx}] = -\Expect[\prob{z=i|\bx}+\Expect[\prob{z=i|\bx}]=0.
\end{align*}
Thus $\bW=\bW^\ast$ is a fixed point for $G(\bW,\bA^\ast)$ since 
\begin{align*}
G(\bW^\ast,\bA^\ast)=\Pi_{\Omega}(\bW^\ast-\alpha \nabla_{\bW^\ast}L_{\log}(\bW^\ast,\bA^\ast))=\bW^\ast.
\end{align*}
Now we make the observation that the population-gradient updates $\bW_{t+1}=G(\bW_t,\bA)$ are same as the gradient-EM updates. Thus the contraction of the population-gradient operator $G(\cdot,\bA^\ast)$ follows from the contraction property of the gradient EM algorithm \cite[Lemma 3]{makkuva2018breaking}. To see this, recall that for $k$-MoE, the gradient-EM algorithm involves computing the function $Q(\bW|\bW_t)$ for the current iterate $\bW_t$ and defined as:
\begin{align*}
Q(\bW|\bW_t)=\Expect \qth{ \sum_{i \in [k-1]} p_{\bW_t}^{(i)} (\bw_i^\top \bx) -\log \pth{1+\sum_{i \in [k-1]} e^{\bw_i^\top \bx} } },
\end{align*}
where $p_{\bW_t}^{(i)}=\prob{z=i|\bx,y,\bw_t}$ corresponds to the posterior probability for the $i^{\mathrm{th}}$ expert, given by
\begin{align*}
p_{\bW_t}^{(i)} = \frac{p_{i,t}(\bx)\calN(y|g(\ba_i^\top \bx),\sigma^2)}{\sum_{j \in [k]}p_{j,t}(\bx)\calN(y|g(\ba_j^\top \bx),\sigma^2)}, \quad p_{i,t}(\bx) = \frac{e^{(\bw_t)_i^\top \bx}}{1+\sum_{j \in [k-1]} e^{(\bw_t)_j^\top \bx}}.
\end{align*}
Then the next iterate of the gradient-EM algorithm is given by $\bW_{t+1}=\Pi_{\Omega}(\bW_t+\alpha \nabla_{\bW}Q(\bW|\bW_t)_{\bW=\bW_t} )$. We have that
\begin{align*}
\nabla_{\bw_i}Q(\bW|\bW_t)|_{\bW=\bW_t} = \Expect\qth{\pth{p_{\bW_t}^{(i)}-\frac{e^{(\bw_t)_i^\top \bx}}{1+\sum_{j \in [k-1]} e^{(\bw_t)_j^\top \bx}}}\bx}=-\nabla_{\bw_i}L_{\log}(\bW_t,\bA).
\end{align*}
Hence if we use the same step size $\alpha$, our population-gradient iterates on the log-likelihood are same as that of the gradient-EM iterates. This finishes the proof. 
\end{proof}

\subsection{Proof of \prettyref{lmm:robustness}}
\begin{proof}
Fix any $\bW \in \Omega$ and let $\bA =\begin{bmatrix}
\ba_1^\top \\ \ldots \\ \ba_k^\top
\end{bmatrix} \in \reals^{k \times d} $ be such that $\max_{i \in [k]}\norm{\ba_i-\ba_i^\ast}_2 =\sigma^2 \varepsilon_1 $ for some $\varepsilon_1>0$. Let 
\begin{align*}
\bW'=G(\bW,\bA), \quad (\bW')^\ast = G(\bW,\bA^\ast).
\end{align*}
Denoting the $i^{\mathrm{th}}$ row of $\bW' \in \reals^{(k-1)\times d}$ by $\bw'_i$ and that of $(\bW')^\ast$ by $(\bw'_i)^\ast$ for any $i \in [k-1]$, we have that
\begin{align*}
\norm{\bw'_i-(\bw'_i)^\ast}_2 &=\norm{\Pi_{\Omega}(\bw_i-\alpha \nabla_{\bw_i}L_{\log}(\bW,\bA))-\Pi_{\Omega}(\bw_i-\alpha \nabla_{\bw_i}L_{\log}(\bW,\bA^\ast))}_2 \\
& \leq \alpha \norm{\nabla_{\bw_i}L_{\log}(\bW,\bA)-\nabla_{\bw_i}L_{\log}(\bW,\bA^\ast) }_2.
\end{align*}
Thus it suffices to bound $\norm{\nabla_{\bw_i}L_{\log}(\bW,\bA)-\nabla_{\bw_i}L_{\log}(\bW,\bA^\ast) }_2$. From \prettyref{eq:gradcomputation}, we have that
\begin{align*}
\nabla_{\bw_i}L_{\log}(\bW,\bA) = -\Expect \qth{\pth{\frac{p_i(\bx)N_i}{\sum_{i \in [k]}p_i(\bx)N_i} - p_i(\bx)}\bx}, \\
 \nabla_{\bw_i}L_{\log}(\bW,\bA^\ast) = -\Expect \qth{\pth{\frac{p_i(\bx)N_i^\ast}{\sum_{i \in [k]}p_i(\bx)N_i^\ast} - p_i(\bx)}\bx},
\end{align*}
where,
\begin{align*}
p_i(\bx) = \frac{e^{\bw_i^\top \bx}}{1+\sum_{k \in [k-1]}e^{\bw_j^\top \bx}}, \quad N_i \define \calN(y|g(\ba_i^\top \bx),\sigma^2), \quad N_i^\ast= \calN(y|g((\ba_i^\ast)^\top \bx),\sigma^2).
\end{align*}
Thus we have
\begin{align}
\norm{\nabla_{\bw_i}L_{\log}(\bW,\bA)-\nabla_{\bw_i}L_{\log}(\bW,\bA^\ast) }_2 = \norm{\Expect[(p^{(i)}(\bA,\bW)-p^{(i)}(\bA^\ast,\bW)  )\bx]}_2,
\label{eq:refer_other}
\end{align}
where $p^{(i)}(\bA,\bW) \define \frac{p_i(\bx)N_i}{\sum_{i \in [k]}p_i(\bx)N_i}$ denotes the posterior probability of choosing the $i^{\mathrm{th}}$ expert. Now we observe that \prettyref{eq:refer_other} reduces to the setting of \cite[Lemma 4]{makkuva2018breaking} and hence the conclusion follows.

\end{proof}

\subsection{Proof of \prettyref{lmm:deviation}}
\begin{proof}
We first prove the lemma for $k=2$. For $2$-MoE, we have that the posterior probability is given by
\begin{align*}
p_{\bw}(\bx,y)=\frac{f(\bw^\top \bx)N_1}{f(\bw^\top \bx)N_1+(1-f(\bw^\top \bx))N_2},
\end{align*}
where $f(\cdot)=\frac{1}{1+e^{-(\cdot)}}, N_1=\calN(y|g(\ba_1^\top \bx),\sigma^2)$ and $N_2=\calN(y|g(\ba_2^\top \bx),\sigma^2)$ for fixed $\ba_1, \ba_2 \in \reals^d$. Then we have that
\begin{align*}
\nabla_{\bw} L_{\log}(\bw,\bA)=- \Expect[(p_{\bw}(\bx,y)-f(\bw^\top \bx))\cdot \bx].
\end{align*}
Hence
\begin{align*}
G(\bw,\bA)= \Pi_{\Omega}(\bw+\alpha \Expect[(p_{\bw}(\bx,y)-f(\bw^\top \bx))\cdot \bx] ), \quad G_n(\bw,\bA)=\Pi_{\Omega}(\bw+\frac{\alpha}{n} \sum_{i\in [n]}(p_{\bw}(\bx_i,y_i)-f(\bw^\top \bx_i))\cdot \bx_i) .
\end{align*}
Since $0<\alpha<1$, we have that
\begin{align*}
\norm{G(\bw,\bA)-G_n(\bw,\bA)}_2 &\leq \|\Expect[(p_{\bw}(\bx,y)-f(\bw^\top \bx)) \bx]-\frac{1}{n} \sum_{i\in [n]}(p_{\bw}(\bx_i,y_i)-f(\bw^\top \bx_i)) \bx_i \|_2 \\
& \leq  \underbrace{\| \Expect[p_{\bw}(\bx,y) \bx]-\sum_{i \in [n]}\frac{p_{\bw}(\bx_i,y_i) \bx_i}{n}
\|_2}_{T_1}+ \underbrace{\| \Expect[f(\bw^\top \bx) \bx]-\sum_{i \in [n]}\frac{f(\bw^\top \bx_i) \bx_i}{n}\|_2}_{T_2}.
\end{align*}
We now bound $T_1$ and $T_2$.

\textbf{Bounding $T_2$:} We prove that the random variable $\sum_{i \in [n]}\frac{f(\bw^\top \bx_i) \bx_i}{n}-\Expect[f(\bw^\top \bx) \bx]$ is sub-gaussian with parameter $L/\sqrt{n}$ for some constant $L >1$ and thus its squared norm is sub-exponential. We then bound $T_2$ using standard sub-exponential concentration bounds. Towards the same, we first show that the random variable $f(\bw^\top \bx) \bx-\Expect[f(\bw^\top \bx)\bx]$ is sub-gaussian with parameter $L$. Or equivalently, that $f(\bw^\top \bx) \inner{\bx}{\bu}-\Expect[f(\bw^\top \bx)\inner{\bx}{\bu}]$ is sub-gaussian for all $\bu \in \mathbb{S}^d$.

Without loss of generality, assume that $\bw \neq 0$. First let $u=\vec{\bw} \define \frac{\bw}{\norm{\bw}}$. Thus $Z \triangleq \inner{\vec{\bw}}{\bx} \sim \calN(0,1)$. We have
\begin{align*}
g(Z) \define f(\bw^\top \bx) \inner{\bx}{\vec{\bw}}-\Expect[f(\bw^\top \bx)\inner{\bx}{\vec{\bw}}]=f(\norm{\bw}Z)Z-\Expect[f(\norm{\bw}Z)Z].
\end{align*}
It follows that $g(\cdot)$ is Lipschitz since 
\begin{align*}
|g'(z)|=|f'(\norm{\bw}z)\norm{\bw}z+f(\norm{\bw}z))| &\leq \sup_{t \in \reals}|f'(t)t|+1  = \sup_{t>0} \frac{te^t}{(1+e^t)^2}+1 \define L.
\end{align*}
From the Talagaran concentration of Gaussian measure for Lipschitz functions \citep{LedouxTal91}, it follows that $g(Z)$ is sub-gaussian with parameter $L$. Now consider any $\bu \in \mathbb{S}^d$ such that $\bu \perp \bw$. Then we have that $Y \define \inner{\bu}{\bx} \sim \calN(0,1)$ and $Z \define \inner{\vec{\bw}}{\bx} \sim \calN(0,1)$ are independent. Thus,
\begin{align*}
g(Y,Z) \define f(\bw^\top \bx) \inner{\bu}{\bx}-\Expect[f(\bw^\top \bx)\inner{\bu}{\bx}]=f(\norm{\bw}Z)Y-\Expect[f(\norm{\bw}Z)Y]
\end{align*}
is sub-gaussian with parameter $1$ since $f \in [0,1]$ and $Y,Z$ are independent standard gaussians. Since any $\bu \in \mathbb{S}^d$ can be written as
\begin{align*}
u = P_{\bw}(u)+P_{\bw^\perp}(u),
\end{align*}
where $P_S$ denotes the projection operator onto the sub-space $S$, we have that  $f(\bw^\top \bx) \inner{\bx}{\bu}-\Expect[f(\bw^\top \bx)\inner{\bx}{\bu}]$ is sub-gaussian with parameter $L$ for all $\bu \in \mathbb{S}^d$. Thus it follows that $\sum_{i \in [n]}\frac{f(\bw^\top \bx_i) \bx_i}{n}-\Expect[f(\bw^\top \bx) \bx]$ is zero-mean and sub-gaussian with parameter $L/\sqrt{n}$ which further implies that
\begin{align*}
T_2 \leq c_2  L \sqrt{\frac{d \log(1/\delta)}{n}},
\end{align*}
with probability at least $1-\delta/2$.

\textbf{Bounding $T_1$:} Let $Z \define \| \sum_{i \in [n]}\frac{p_{\bw}(\bx_i,y_i) \bx_i}{n}-\Expect[p_{\bw}(\bx,y) \bx]
\|_2 = \sup_{\bu \in \mathbb{S}^d}Z(\bu)$, where 
\begin{align*}
Z(\bu) \define \sum_{i \in [n]}\frac{p_{\bw}(\bx_i,y_i) \inner{\bx_i}{\bu}}{n} -\Expect[p_{\bw}(\bx,y) \inner{\bx}{\bu}].
\end{align*}
Let $\{\bu_1,\ldots,\bu_M\}$ be a $1/2$-cover of the unit sphere $\mathbb{S}^d$. Hence for any $\bv \in \mathbb{S}^d$, there exists a $j \in [M]$ such that $\norm{\bv-\bu_j}_2 \leq 1/2$. Thus,
\begin{align*}
Z({\bv}) \leq Z(\bu_j)+|Z(\bv)-Z(\bu_j)| \leq Z \norm{\bv-\bu_j}_2 \leq Z(\bu_j)+Z/2,
\end{align*}
where we used the fact that $|Z(\bu)-Z(\bv)|\leq Z\norm{\bu-\bv}_2$ for any $\bu,\bv \in \mathbb{S}^d$. Now taking supremum over all $\bv \in \mathbb{S}^d$ yields that $Z \leq 2 \max_{j \in [M]}Z(\bu_j)$. Now we bound $Z(\bu)$ for a fixed $\bu \in \mathbb{S}^d$. By symmetrization trick \citep{vaart1996weak}, we have
\begin{align*}
\prob{Z(\bu) \geq t} \leq 2\prob{\frac{1}{n}\sum_{i=1}^n \varepsilon_i p_{\bw}(\bx_i,y_i)\inner{\bx_i}{\bu} \geq t/2},
\end{align*}
where $\varepsilon_1,\ldots,\varepsilon_n$ are \iid Rademacher variables. Define the event $E\define \{\frac{1}{n}\sum_{i \in [n]}\inner{\bx_i}{\bu}^2 \leq 2\}$. Since $\inner{\bx_i}{\bu} \sim \calN(0,1)$, standard tail bounds imply that $\prob{E^c} \leq e^{-n/32}$. Thus we have that
\begin{align*}
\prob{Z(\bu) \geq t} \leq 2\prob{\frac{1}{n}\sum_{i=1}^n \varepsilon_i p_{\bw}(\bx_i,y_i)\inner{\bx_i}{\bu} \geq t/2|E}+2e^{-n/32}.
\end{align*}
Considering the first term, for any $\lambda>0$, we have
\begin{align*}
\Expect[ \exp \pth{\frac{\lambda}{n}\sum_{i=1}^n \varepsilon_i p_{\bw}(\bx_i,y_i)\inner{\bx_i}{\bu}} | E] \leq \Expect[ \exp \pth{\frac{2\lambda}{n}\sum_{i=1}^n \varepsilon_i \inner{\bx_i}{\bu}} | E],
\end{align*}
where we used the Ledoux-Talagrand contraction for Rademacher process \citep{LedouxTal91}, since $|p_{\bw}(\bx_i,y_i)| \leq 1$ for all $(\bx_i,y_i)$. The sub-gaussianity of Rademacher sequence $\{\varepsilon_i\}$ implies that
\begin{align*}
\Expect[ \exp \pth{\frac{2\lambda}{n}\sum_{i=1}^n \varepsilon_i \inner{\bx_i}{\bu}} | E] \leq \Expect[\exp \pth{\frac{2\lambda^2}{n^2}\sum_{i=1}^n \inner{\bx_i}{\bu}^2 } | E ] \leq \exp(\frac{4\lambda^2}{n}),
\end{align*}
using the definition of the event $E$. Thus the above bound on the moment generating function implies the following tail bound:
\begin{align*}
\prob{\frac{1}{n}\sum_{i=1}^n \varepsilon_i p_{\bw}(\bx_i,y_i)\inner{\bx_i}{\bu} \geq t/2|E} \leq \exp\pth{-\frac{nt^2}{256}}.
\end{align*}
Combining all the bounds together, we obtain that
\begin{align*}
\prob{Z(\bu)\geq t} \leq 2 e^{-nt^2/256}+2e^{-n/32}.
\end{align*}
Since $M \leq 2^d$, using the union bound we obtain that
\begin{align*}
\prob{Z \geq t} \leq 2^d(2 e^{-nt^2/1024}+2e^{-n/32}).
\end{align*}
Since $n \geq c_1 d \log (1/\delta)$, we have that $T_1=Z \leq c \sqrt{\frac{d \log(1/\delta)}{n}}$ with probability at least $1-\delta/2$. Combining these bounds on $T_1$ and $T_2$ yields the final bound on $\varepsilon_G(n,\delta)$.

Now consider any $k \geq 2$. From \prettyref{eq:gradcomputation}, defining $N_i \define \calN(y|g(\ba_i^\top \bx),\sigma^2)$ and $p_i(\bx)=\frac{e^{\bw_i^\top \bx}}{1+\sum_{j \in [k-1]}e^{\bw_j^\top \bx} }$, we have that
\begin{align*}
\nabla_{\bw_i}L_{\log}(\bW,\bA) = -\Expect \pth{\frac{p_i(\bx)N_i}{\sum_{i \in [k]}p_i(\bx)N_i} - p_i(\bx)}\bx.
\end{align*}
Similarly,
\begin{align*}
\nabla_{\bw_i}L^{(n)}_{\log}(\bW,\bA) = -\sum_{j=1}^n \frac{1}{n} \pth{\frac{p_i(\bx_j)N_i}{\sum_{i \in [k]}p_i(\bx_j)N_i} - p_i(\bx_j)}\bx_j.
\end{align*}
Since $\norm{G_n(\bW,\bA)-G(\bW,\bA)}=\max_{i \in [k-1]}\norm{G_n(\bW,\bA)_i-G(\bW,\bA)_i}_2$, with out loss of generality, we let $i=1$. The proof for the other cases is similar. Thus we have
\begin{align*}
\norm{G_n(\bW,\bA)_1-G(\bW,\bA)_1}_2 &\leq \norm{\nabla_{\bw_1}L_{\log}(\bW,\bA) -\nabla_{\bw_1}L^{(n)}_{\log}(\bW,\bA) }_2 \\
&\leq \norm{\sum_{i=1}^n \frac{p^{(1)}(\bx_i,y_i)\bx_i}{n}-\Expect[p^{(1)}(\bx,y)\bx]}_2+ \norm{\sum_{i=1}^n \frac{p_1(\bx)\bx}{n}-\Expect[p_1(\bx)\bx]}_2,
\end{align*}
where $p^{(1)}(\bx,y) \define \frac{p_1(\bx)N_1}{\sum_{i \in [k]}p_i(\bx)N_i}$. Since $|p^{(1}(\bx,y)| \leq 1$ and $|p_1(\bx)| \leq 1$, we can use the same argument as in the bounding of $T_1$ proof for $2$-MoE above to get the parametric bound. This finishes the proof. 
\end{proof}

\section{Additional experiments}
\label{app:additionalexp}

\subsection{Reduced batch size}

\begin{figure*}[t]
    \centering
    \begin{subfigure}[b]{0.45\textwidth}
        \includegraphics[width=\textwidth]{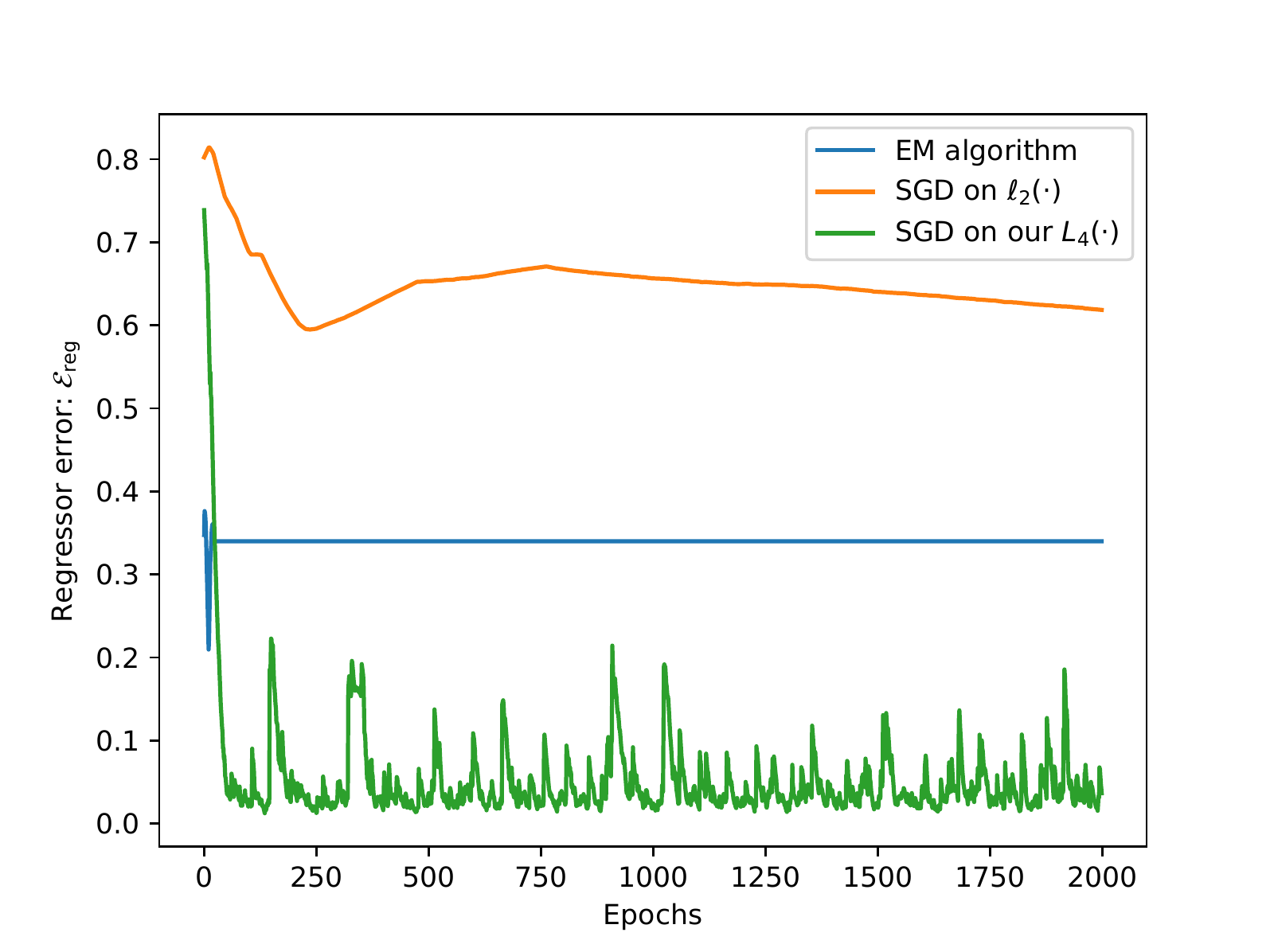}
        \caption{Regressor error}
        \label{fig:all_regressor_error_appendix}
    \end{subfigure}
    \begin{subfigure}[b]{0.45\textwidth}
        \includegraphics[width=\textwidth]{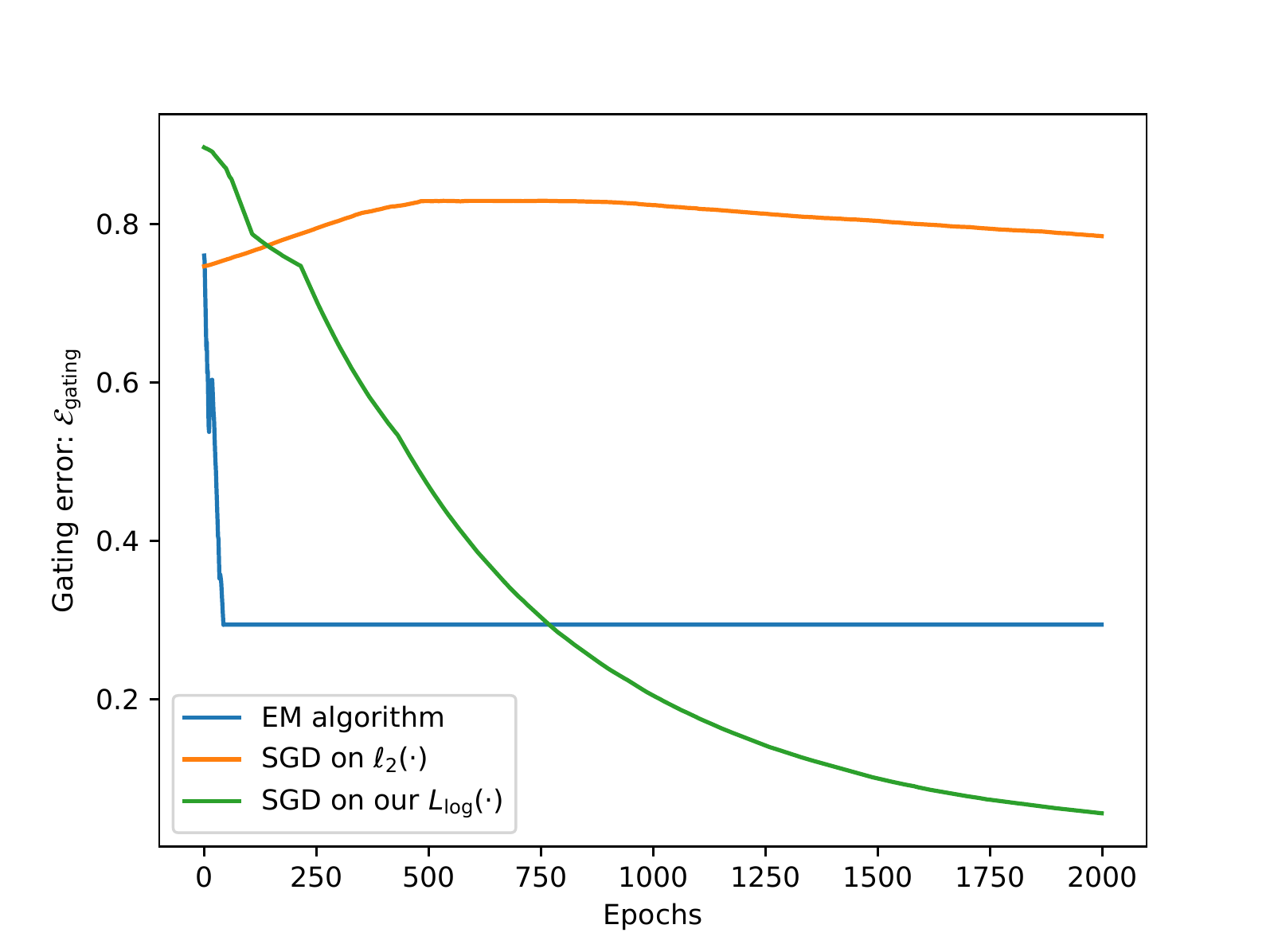}
        \caption{Gating error}
        \label{fig:all_classifier_error_appendix}
    \end{subfigure}
    ~ 

\caption{Comparison of SGD on our losses $(L_4, L_{\mathrm{log}})$ vs. $\ell_2$ and the EM algorithm. }
\label{fig:regresslearnappendix}
\end{figure*}

In \prettyref{fig:regresslearnappendix} we ran SGD on our loss $L_4(\cdot)$ with $5$ different runs with a batch size of $128$ and a learning rate of $0.001$ for $d=10$ and $k=3$. We can see that our algorithm still converges to zero but with a more variance because of noisy gradient estimation and also lesser number of samples than the required sample complexity.

%
%
%

\end{document}